\documentclass{article}

\PassOptionsToPackage{numbers, compress}{natbib}

    \usepackage[preprint]{neurips_2025}

\usepackage[utf8]{inputenc} 
\usepackage[T1]{fontenc}    
\usepackage{hyperref}       
\usepackage{url}            
\usepackage{booktabs}       
\usepackage{amsfonts}       
\usepackage{nicefrac}       
\usepackage{microtype}      
\usepackage{xcolor}         

\usepackage{hyperref}
\usepackage{url}
\usepackage{xcolor}
\usepackage{graphicx}
\usepackage{multirow}
\usepackage{pdfpages}
\usepackage[ruled,vlined]{algorithm2e}
\usepackage{algorithmic,amsmath}
\usepackage{wrapfig}
\usepackage{amssymb}
\usepackage{enumitem}
\usepackage{transparent}
\usepackage{booktabs}
\usepackage{graphicx}
\usepackage{amsmath,amssymb,amsthm}
\usepackage{hyperref}
\usepackage{dsfont}

\newtheorem{theorem}{Theorem}

\usepackage{ulem}
\usepackage{floatrow}
\usepackage{subcaption}
\newfloatcommand{capbtabbox}{table}[][\FBwidth]
\usepackage{blindtext}
\usepackage{listings}
\usepackage{tcolorbox} 
\tcbuselibrary{listings,skins} 
\definecolor{custompurple}{rgb}{0.408, 0.2, 0.6}  
\definecolor{customgreen}{rgb}{0.627, 0.808, 0.388}  
\definecolor{customblue}{rgb}{0.216, 0.329, 0.573}  

\newtcolorbox{cvbox}[1][]{
    enhanced,
    after skip=8mm,
    title=#1,
    fonttitle=\sffamily\bfseries\color{white},
    coltitle=white,
    colbacktitle=black!75!white, 
    titlerule= 0pt,         
    overlay={%
        \ifcase\tcbsegmentstate
        \or%
        \else%
        \fi%
    }
    colback = gray!5!white,         
    colframe = black!75     
    }

\title{Detection and Mitigation of Hallucination in Large Reasoning Models: A Mechanistic Perspective}

\author{%
  Zhongxiang Sun  \quad Qipeng Wang \quad Haoyu Wang \quad
  Xiao Zhang \quad Jun Xu\thanks{Corresponding author. Work partially done at Engineering Research Center of Next-Generation Intelligent Search and Recommendation, Ministry of Education.} \\
  Gaoling School of Artificial Intelligence\\ Renmin University of China \\
  Beijing, China \\
  \texttt{\{sunzhongxiang\}@ruc.edu.cn}
}

\begin{document}

\maketitle

\begin{abstract}
Large Reasoning Models (LRMs) have shown impressive capabilities in multi-step reasoning tasks. However, alongside these successes, a more deceptive form of model error has emerged—\textit{Reasoning Hallucination}—where logically coherent but factually incorrect reasoning traces lead to persuasive yet faulty conclusions. Unlike traditional hallucinations, these errors are embedded within structured reasoning, making them more difficult to detect and potentially more harmful.
In this work, we investigate reasoning hallucinations from a mechanistic perspective. We propose the \textbf{Reasoning Score}, which quantifies the depth of reasoning by measuring the divergence between logits obtained from projecting late layers of LRMs to the vocabulary space, effectively distinguishing shallow pattern-matching from genuine deep reasoning. Using this score, we conduct an in-depth analysis on the ReTruthQA dataset and identify two key reasoning hallucination patterns: early-stage fluctuation in reasoning depth and incorrect backtracking to flawed prior steps. These insights motivate our \textbf{R}easoning \textbf{H}allucination \textbf{D}etection (\textbf{RHD}) framework, which achieves state-of-the-art performance across multiple domains.
To mitigate reasoning hallucinations, we further introduce \textbf{GRPO-R}, an enhanced reinforcement learning algorithm that incorporates step-level deep reasoning rewards via potential-based shaping. Our theoretical analysis establishes stronger generalization guarantees, and experiments demonstrate improved reasoning quality and reduced hallucination rates. 

\end{abstract}

\section{Introduction}
\label{sec:intro}

Hallucination has long been a critical safety challenge for Large Language Models (LLMs). In this context, hallucination refers to outputs that appear fluent and coherent but are semantically inaccurate or lack factual grounding. With the advent of {Large Reasoning Models (LRMs)}—such as DeepSeek-R1~\citep{DeepSeekAI2025} and OpenAI's O-series~\citep{openai2025o3o4}—AI systems have demonstrated unprecedented potential in solving complex real-world tasks. These models are typically trained with outcome-based reinforcement learning (RL) and explicitly generate multi-step reasoning traces prior to final answers.

Recent studies have uncovered a subtler form of hallucination emerging in LRMs~\citep{transluce2024o3, vectara2025deepseek, openai2025o3o4}, which we refer to as \textbf{Reasoning Hallucination}. Unlike traditional hallucinations, reasoning hallucinations are often embedded within logically coherent reasoning traces, making incorrect information more persuasive and harder to detect. This form of ``plausible but incorrect'' reasoning can elicit user trust, resembling the {conjunction fallacy}, where detailed yet misleading explanations are perceived as more credible than simpler ones~\citep{tentori2004conjunction,valmeekam2024llms}.
Prior studies mainly assess the correctness of reasoning paths in standard Chain-of-Thought (CoT) tasks over relatively simple problems~\citep{xu2024can, prasad2023receval}, with limited investigation into the mechanisms of hallucinations in LRMs. Recent work has extended evaluation to long CoT generated by LRMs~\citep{processbench}, yet remains focused on error identification rather than uncovering underlying causes.
However, directly analyzing model-generated traces can be misleading due to the subtle nature of reasoning hallucinations. The emergence of Latent CoT, where reasoning is embedded in hidden states rather than surface text, further obscures detection~\citep{hao2024training}. These challenges call for probing the internal mechanisms behind reasoning hallucinations, enabling interpretable and robust hallucination detection.

Recent studies on the reasoning capabilities of LRMs~\citep{mirzadeh2024gsm, yan2025recitation} have shown that models often produce incorrect answers when their reasoning process relies on shallow pattern-matching rather than genuine deep reasoning. This mirrors findings in cognitive science, where human thinking patterns are closely linked to the emergence of cognitive illusions ~\citep{kahneman2011thinking, bruckmaier2021tversky, weis2024switching}. Inspired by these observations, we investigate reasoning hallucinations in LRMs through the lens of internal thinking patterns, where a central challenge is how to quantify whether a model is performing deep reasoning or merely matching surface-level patterns from training data. Prior mechanistic interpretability studies highlight a functional division within language models: early layers primarily transmit information, while later layers perform more complex reasoning over aggregated context~\citep{nikankin2025arithmetic, chen2025bring}. Based on this insight, we introduce \textbf{Reasoning Score}, which measures the divergence between logits obtained from projecting late layers of LRMs to the vocabulary space. Through synthetic experiments, we validate the effectiveness of the Reasoning Score in measuring the depth of reasoning in LRMs, which reflects whether the model engages in shallow pattern-matching or deep reasoning (\S~\ref{sec:reasoning_score}).

Building on the proposed reasoning score, we conduct extensive analyses on reasoning hallucinations using the ReTruthQA dataset. We identify three key patterns of reasoning hallucination: \textbf{Pattern \#1}: large fluctuations in reasoning depth during the early steps, and \textbf{Pattern \#2}: incorrect backtracking from later steps to earlier incorrect steps. We attribute these patterns to the presence of {shallow pattern-matching} and {overthinking} steps, which undermine the LRM’s inherent abilities in self-verification and backtracking, ultimately leading to reasoning hallucinations (\S~\ref{sec:reason_analysis}). Moreover, we observe that \textbf{Pattern \#3:} overthinking steps exhibit a positive correlation between reasoning scores and perplexity, indicating spurious verification behaviors (\S~\ref{sec:under_pattern1}). 
Based on these findings, we design the \textbf{R}easoning \textbf{H}allucination \textbf{D}etection (\textbf{RHD}) method, which significantly outperforms baselines across diverse domains in the reasoning hallucination detection dataset (\S~\ref{sec:rhd_method}).

We further investigate the underlying cause of shallow pattern-matching and overthinking steps in LRMs and attribute it to the outcome-based RL paradigm commonly used during training. This paradigm incentivizes correct final answers but neglects whether intermediate reasoning steps reflect deep and meaningful thinking. To address this challenge, we introduce a {step-level deep reasoning reward} based on the reasoning score and propose \textbf{GRPO-R}, a variant of Group Relative Policy Optimization (GRPO)~\citep{shao2024deepseekmath,DeepSeekAI2025} that incorporates potential-based reward shaping. GRPO-R encourages deep—but not excessive—reasoning during RL fine-tuning. Our theoretical analysis shows that GRPO-R leads to better generalization in outcome-based RL, and empirical results confirm that it improves reasoning accuracy compared to standard GRPO (\S~\ref{sec:mitigation}).


\section{Related Works}

\textbf{Hallucination of Language Models.}
Hallucination remains a fundamental safety concern for LLMs, and outcome-supervised LRMs~\citep{DeepSeekAI2025,openai2025o3o4} exacerbate this issue by generating logically flawed but persuasive reasoning traces, a consequence of reward-seeking behavior induced by outcome-based RL without step-level supervision~\citep{chen2025reasoning,valmeekam2024llms,transluce2024o3}. Detection approaches span uncertainty estimation~\citep{kadavath2022language,malinin2020uncertainty,ren2022out}, internal signal probing~\citep{cheninside,li2025languagemodeluncertaintyquantification,li2024inference}, process-level critique models~\citep{he2025largelanguagemodelsdetect}, and Process Reward Models (PRMs)~\citep{prmlessons}, though challenges remain due to the deceptive nature of hallucinated traces and the poor generalization of PRM signals~\citep{zheng2024processbenchidentifyingprocesserrors}. 
We address this by conducting a mechanistic analysis of reasoning hallucinations and proposing a detection method grounded in internal model behavior.

\textbf{Mechanistic Interpretability.}
Mechanistic interpretability~\citep{ferrando2024primer, transformercircuits2021} seeks to explain model behavior by attributing predictions to internal components. In transformers, attention heads contextualize token representations~\citep{ferrando2024information, wu2024retrieval}, while FFNs act as knowledge storage~\citep{geva2021transformer}. Recent work has applied intervention-based techniques from mechanistic interpretability to analyze how LLMs perform reasoning, revealing a functional division of labor across layers in various tasks such as math reasoning and multimodal reasoning: early layers primarily transmit contextual information, and the reasoning process is predominantly carried out by the later layers.~\citep{chen2025bring, nikankin2025arithmetic, li2024understanding}.  These insights motivate our design of the \textit{Reasoning Score}, which captures thinking patterns by quantifying hidden state shifts in later layers, laying the groundwork for analyzing reasoning hallucinations in LRMs.

\section{Empirical Study of Reasoning Hallucination}
\label{sec:emp_study}
Our empirical study investigates the relationship between reasoning hallucinations and the thinking patterns of LRMs, where thinking patterns are quantified using a reasoning score derived from mechanistic interpretability. This analysis reveals key reasoning hallucination patterns and guides the design of more effective detection and mitigation strategies.


\begin{figure}[t]
    \centering
    \includegraphics[width=0.98\textwidth]{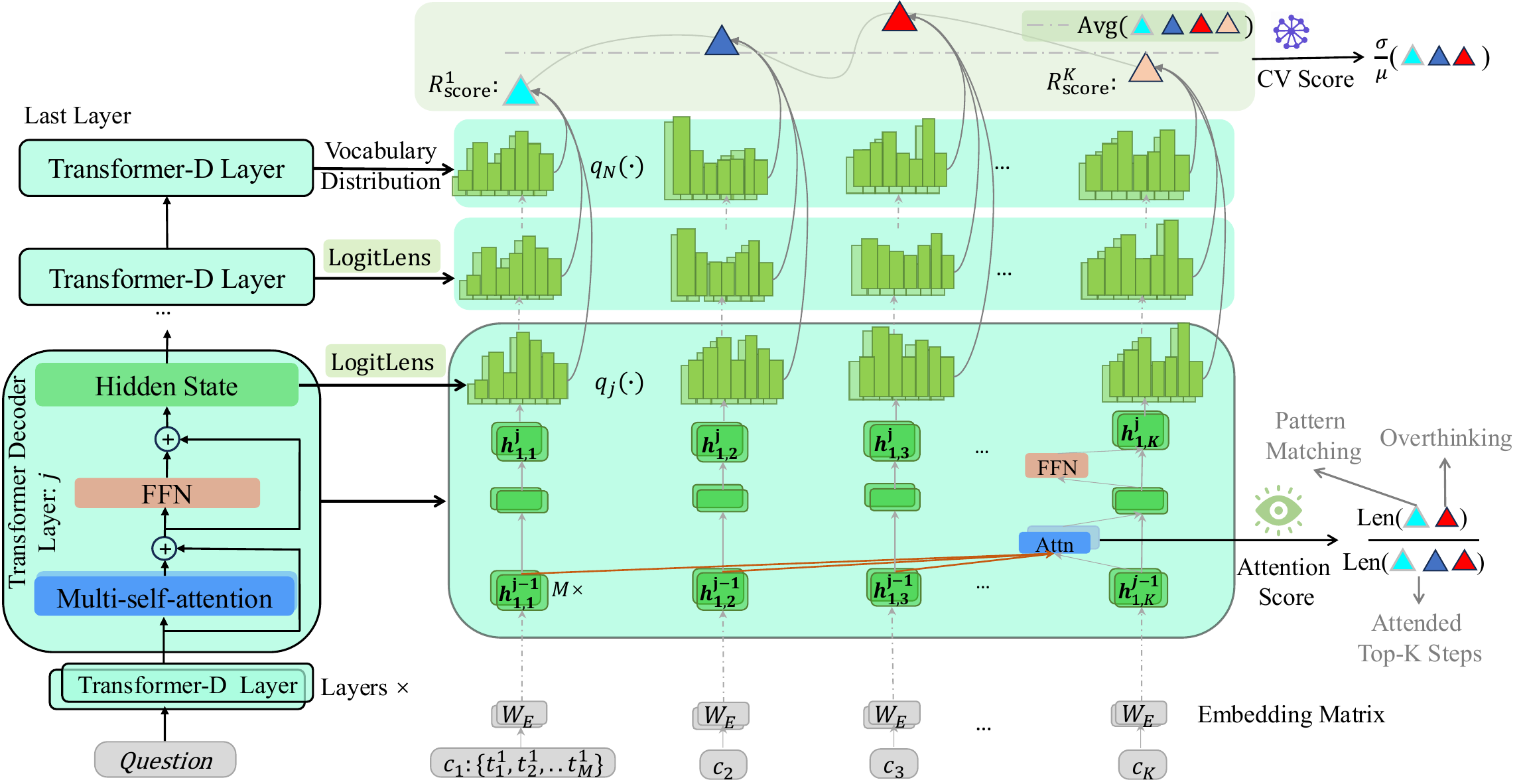}
    \caption{The illustration of the calculation processes for the Reasoning Score (Eq.~\ref{eq:r_score}), CV Score (Eq.~\ref{eq:cv_score}), and Attention Score (Eq.~\ref{eq:attn_score}).}
    \label{fig:method}
\end{figure}

\subsection{Reasoning Score: Measuring Reasoning Depth in Large Reasoning Model}
\label{sec:reasoning_score}
To determine whether a reasoning step is generated via shallow pattern matching or genuine deep reasoning, we propose a \textit{Reasoning Score} inspired by mechanistic interpretability. Prior studies analyzing the internal mechanisms of language models reveal a layered functional division: early layers primarily transmit information, while later layers perform more complex reasoning over aggregated context to produce correct outputs~\citep{stolfo2023a, nikankin2025arithmetic, li2024understanding}. Building on this insight, we define the reasoning score under the hypothesis that deeper reasoning is reflected by meaningful transformations in later-layer representations during generation.

Formally, a LRM-generated reasoning trace \( C = [c_1, c_2, \dots, c_K] \) consists of multiple reasoning steps, each associated with a step-level reasoning score \( R^{k}_{\text{score}} \) that quantifies the depth of reasoning in step \( c_k \). Each reasoning step \(c_k = \langle t^{k}_1, \ldots, t^{k}_M \rangle\) is composed of \(M\) tokens. The overall reasoning trace score $\mathcal{R}_{\text{score}}$ is represented as a sequence \( [R^{1}_{\text{score}}, R^{2}_{\text{score}}, \dots, R^{K}_{\text{score}}] \), capturing the model's reasoning dynamics across steps. As shown in \autoref{fig:method}, each score is defined as the mean Jensen–Shannon divergence (JSD) between vocabulary distributions induced by hidden states from selected later layers and the anchor distribution from the final layer.
To obtain the output distribution from each token hidden state \( h_{m, k}^{(j)} \) of token $t_{m}^{k}$ at layer $j$, we apply the $\operatorname{LogitLens}$~\citep{nostalgebraist2020}, which projects each layer-normalized hidden state into vocabulary space via the unembedding matrix \( \boldsymbol{W}_U \): 
$
\operatorname{LogitLens}(h_{m,k}^{(j)}) = \operatorname{LayerNorm}(h_{m,k}^{(j)}) \boldsymbol{W}_U.
$
This provides a layer-wise interpretation of token prediction behavior and has been widely adopted for interpreting LLM internal representations~\citep{hanna2024does,zhou2024unibias,yu2023characterizing}.

The final step-level Reasoning Score $R^{k}_{\text{score}}$ is computed as:
\begin{align} 
R^{k}_{\text{score}} &= \frac{1}{|c_k|} \sum_{t_{m+1}^{k} \in c_k} \frac{1}{|\mathcal{J}|} \sum_{j \in \mathcal{J}} \text{JSD}\left(q_N\left(t_{m+1}^{k}\right), q_j\left(t_{m+1}^{k}\right)\right), \\
q_j\left(t_{m+1}^{k}\right) &= \operatorname{softmax}\left(\operatorname{LogitLens}\left(h_{m,k}^{(j)}\right)\right),
\quad j \in \mathcal{J},
\label{eq:r_score}
\end{align}
where \( \mathcal{J} \) denotes the set of selected later layers and \( q_N \) is the anchor distribution from the final layer.

Intuitively, a larger score \( R_{\text{score}} \) indicates substantial transformation in output distributions within late layers, suggesting the model is actively engaging in deep reasoning by integrating earlier contextual information. In contrast, a smaller score implies distributional stability in late layers, indicating shallow pattern matching or heuristic-based processing without further reasoning, consistent with prior findings on the differential roles of early versus later layers.

\begin{figure}[t]
    \centering
    \includegraphics[width=0.95\textwidth]{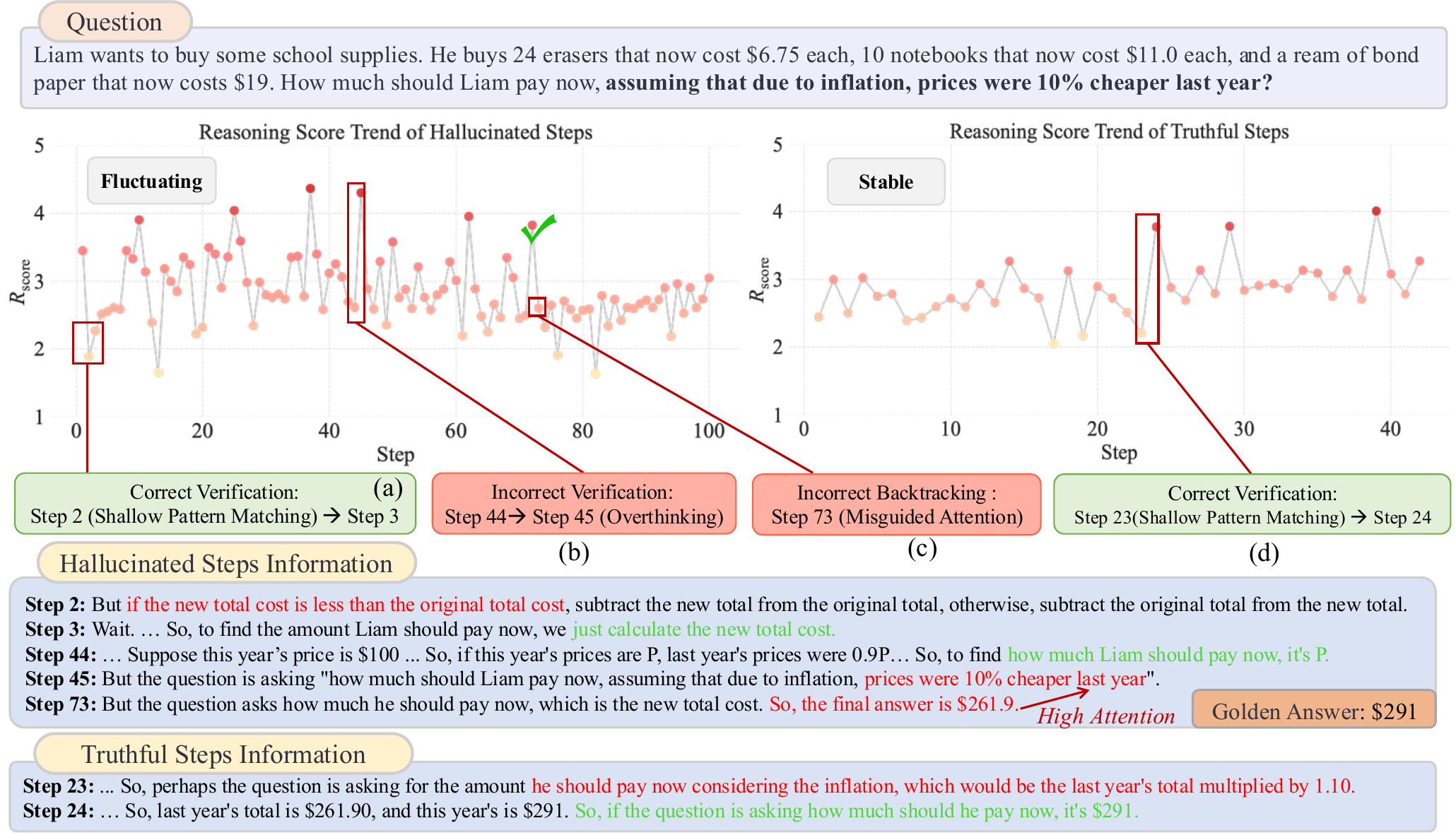}
    \caption{Case study from GSM-NoOp dataset~\cite{mirzadeh2024gsm} on R1-7B. We sample both a hallucinated reasoning trace (left) and a truthful reasoning trace (right) for the same question as a preliminary analysis of reasoning hallucinations. Reasoning scores are scaled by $1\mathrm{e}{5}$.}
    \label{fig:case_study}
\end{figure}

\paragraph{Validating the Reasoning Score with GSM-NoOp.}

We validate whether the Reasoning Score faithfully reflects reasoning depth using GSM-NoOp~\citep{mirzadeh2024gsm}, a GSM8K-derived dataset where semantically irrelevant but plausible \texttt{No-Op} phrases are injected into problems. Although these phrases do not alter the correct reasoning path, prior work shows that LRMs are often misled by them, revealing their reliance on shallow pattern matching~\citep{mirzadeh2024gsm}.
This makes GSM-NoOp a suitable testbed: if the Reasoning Score captures reasoning depth, then steps misled by No-Op phrases should yield lower scores. We validate this using correct outputs from \texttt{DeepSeek-R1-Distill-Qwen-7B (R1-7B)} to avoid confounds from hallucinated traces. Misled steps are labeled via GPT-4o. As GSM-NoOp is not publicly available, we re-implement a compatible version following the original paper’s methodology, with prompts and details provided in Appendix~\ref{app:GSM-NoOp}.

\paragraph{Results.}
Our empirical results in \autoref{fig:emp_1_result} (a) show that reasoning steps misled by No-Op phrases consistently receive significantly lower Reasoning Scores compared to non-misled steps. This supports our hypothesis that the Reasoning Score effectively captures shallow pattern-matching behavior and serves as an indicator of whether a model is engaging in deep reasoning.

\subsection{Reasoning Hallucination Analysis Based on Reasoning Score}
\label{sec:reason_analysis}
In this section, we leverage the mechanistically derived Reasoning Score as a proxy for the thinking patterns of LRMs and investigate its relationship with the emergence of reasoning hallucinations. We begin with a preliminary analysis to identify characteristic patterns associated with hallucinated reasoning traces. We then analyze the generality of these patterns across domains using the ReTruthQA dataset, and further examine the underlying mechanism that leads LRMs to exhibit such behaviors.

\subsubsection{Case Analysis on GSM-NoOps}
\label{sec:case_study}
In this section, we conduct a preliminary analysis using the LRM R1-7B on a question from GSM-NoOp~\citep{mirzadeh2024gsm}, where a ``NoOp'' statement is appended to the end of a math problem. To enable controlled comparison of reasoning hallucination patterns, we sample both a truthful and a hallucinated response from R1-7B on the same question. Figure~\ref{fig:case_study} presents the question along with step-level reasoning scores \( R_{\text{score}} \), which quantify the depth of thinking at each step.

We observe that when the model generates reasoning steps that attend to the added NoOp content, these steps typically receive lower \( R_{\text{score}} \), which in turn triggers the model’s \textit{Self-Verification} mechanism~\citep{li2025llms}, producing later steps with higher \( R_{\text{score}} \) that attempt to correct the earlier deviation (e.g., (a) and (d) in \autoref{fig:case_study}). However, in the hallucinated reasoning trace, we also observe \textit{overthinking} phenomena—steps with excessively high \( R_{\text{score}} \) that incorrectly revise the previous correct reasoning steps (e.g., (b) in \autoref{fig:case_study}). These hallucinated traces contain more shallow pattern-matching and overthinking steps, resulting in an overall unstable reasoning trajectory.
From this case study, we identify the reasoning hallucination \textbf{Pattern \#1}: hallucinated traces typically exhibit large fluctuations in reasoning score, especially during the early steps of the process. 

Furthermore, we observe that even when the model briefly arrives at correct intermediate steps, it often fails to maintain this correctness. In later steps, it performs \textit{Incorrect Backtracking}, attending to earlier shallow or overthinking steps, ultimately leading to hallucination (e.g., (c) in \autoref{fig:case_study}). This motivates the reasoning hallucination \textbf{Pattern \#2}: in the later stages of reasoning, the model tends to misguidedly attend to earlier hallucinated steps, either shallow or overthinking, making it difficult to correct earlier errors and leading to hallucinated reasoning. 



\begin{figure}[t]
    \centering
    \begin{subfigure}{0.33\textwidth}
        \includegraphics[width=\linewidth]
        {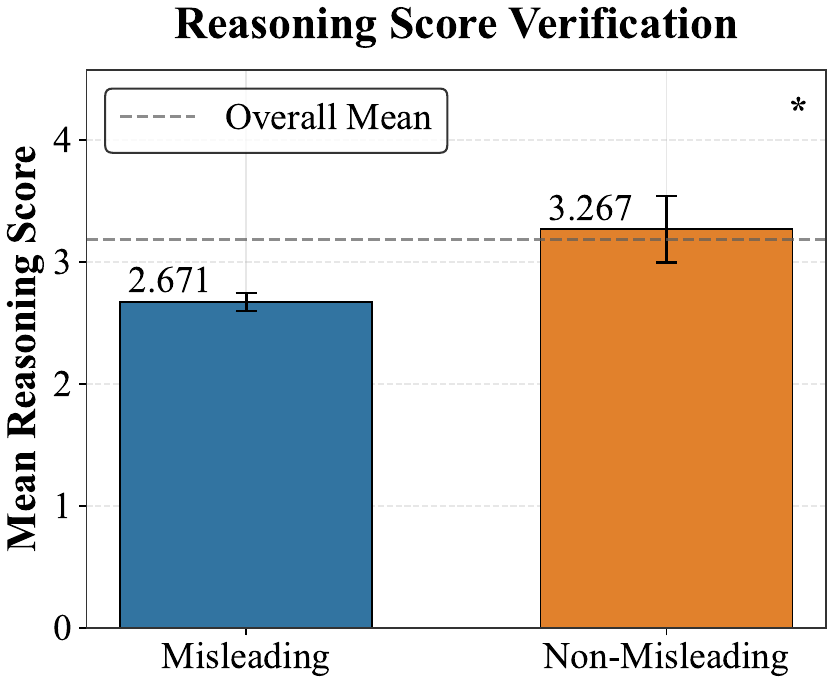}
        \caption*{(a)}
    \end{subfigure}
    \hfill
    \begin{subfigure}{0.31\textwidth}
        \includegraphics[width=\linewidth]{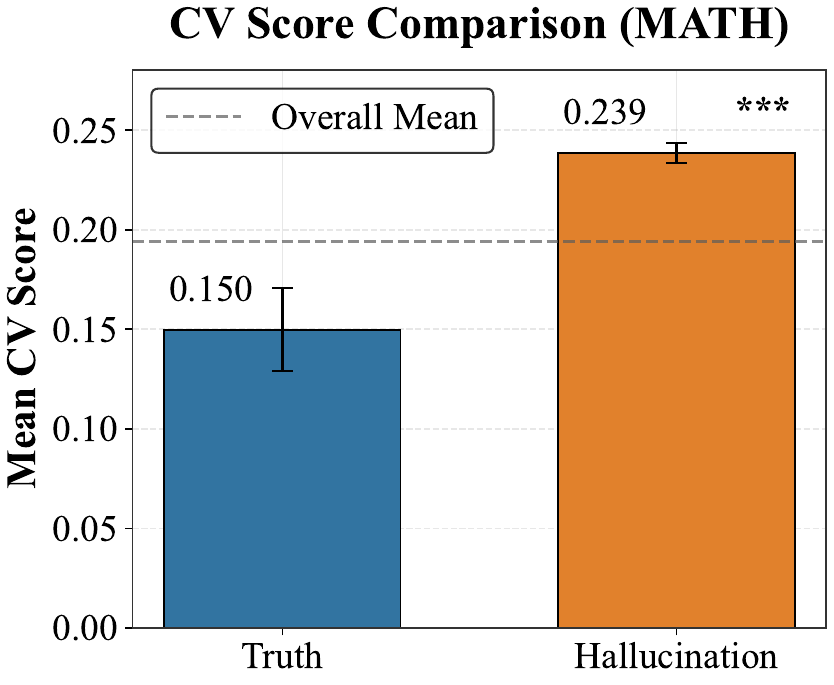}
        \caption*{(b)}
    \end{subfigure}
    \hfill
    \begin{subfigure}{0.32\textwidth}
        \includegraphics[width=\linewidth]{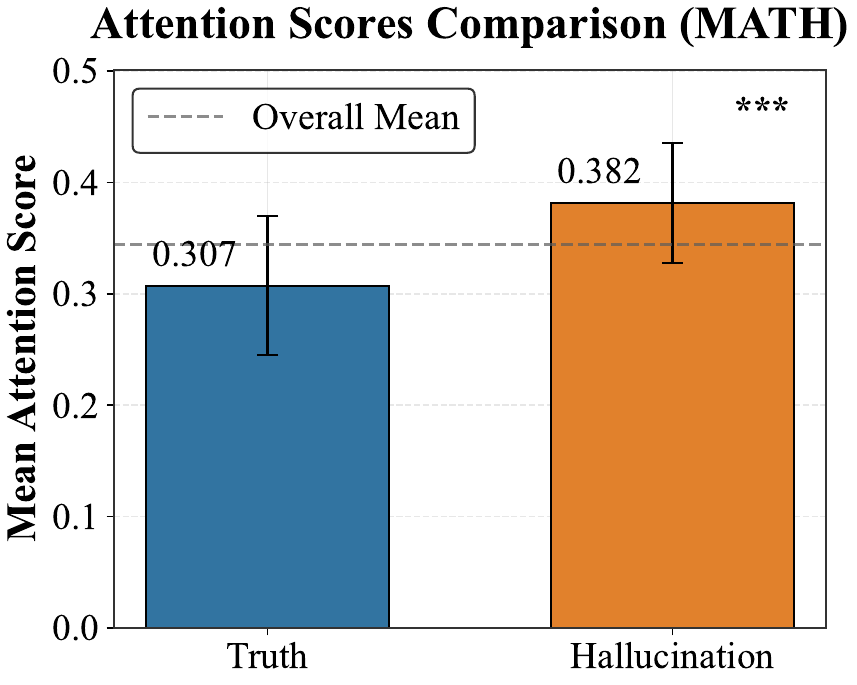}
        \caption*{(c)}
    \end{subfigure}
    \caption{(a) Reasoning Score validation on GSM-NoOp. (b) Evaluation of Pattern~\#1 (early fluctuations), and (c) Pattern~\#2 (misguidedly attention) on ReTruthQA. Asterisks indicate statistical significance based on a t-test: * for $p$-value < 0.05, and *** for $p$-value < 0.001.}
    \label{fig:emp_1_result}
\end{figure}

\subsubsection{Reasoning Hallucination Pattern Analysis}

In this section, we validate the two reasoning hallucination patterns identified in preliminary analysis(\S~\ref{sec:case_study}): \textbf{Pattern \#1}: large fluctuations in reasoning scores during early steps, and \textbf{Pattern \#2}: incorrect backtracking to earlier hallucinated reasoning steps in later stages. We aim to assess whether these patterns generalize across broader domains and tasks. To this end, we conduct experiments on the ReTruthQA dataset using the R1-7B model. ReTruthQA covers three reasoning domains: \texttt{Math}, \texttt{Science}, and \texttt{MultiHopQA} (Details in \S~\ref{sec:rhd_dataset}). For each domain, we construct two balanced subsets using gold hallucination labels: one with hallucinated traces and one with truthful traces.

To evaluate \textbf{Pattern \#1}, we measure the fluctuation of reasoning depth in the early phase of reasoning using the Coefficient of Variation (\textbf{CV Score})~\citep{everitt1998cambridge}, a standard metric for quantifying sequence variability (shown in~\autoref{fig:method}). Specifically, we focus on the first \( \lceil K / r \rceil \) steps of the reasoning trace \( \mathcal{C} = \langle c_1, c_2, \dots, c_K \rangle \), and define:
$
\mathcal{R}_{\text{score}}^{\text{early}} = \left[ R^1_{\text{score}}, R^2_{\text{score}}, \dots, R^{\lceil K/r \rceil}_{\text{score}} \right],
$
where \( r > 1 \) is a constant controlling the size of the early-step window. The CV score over early reasoning steps is then given by:
\begin{equation}
\text{CV}(\mathcal{C}) = \frac{\sigma(\mathcal{R}_{\text{score}}^{\text{early}})}{\mu(\mathcal{R}_{\text{score}}^{\text{early}})},
\label{eq:cv_score}
\end{equation}
where \( \mu(\cdot) \) and \( \sigma(\cdot) \) denote the mean and standard deviation, respectively.

To assess \textbf{Pattern \#2}, we introduce a \textbf{Attention Score} that quantifies the extent to which later reasoning steps attend to earlier shallow-pattern matching or overthinking steps (\autoref{fig:method}). Let the full reasoning trace be \( \mathcal{C} = \langle c_1, c_2, \dots, c_K \rangle \), and define the later reasoning steps as \( \mathcal{C}_{\text{later}} = \{c_k\}_{k=\lceil \eta K \rceil}^{K} \).
For a step \( c_k \in \mathcal{C}_{\text{later}} \), we compute the mean attention from \( c_k \) to each earlier step \( c_j \) as:
\[
\bar{a}_{k \to j} = \frac{1}{|c_k||c_j|} \sum_{t \in c_k} \sum_{s \in c_j} \left( \frac{1}{|\mathcal{L}|} \sum_{l \in \mathcal{L}} \frac{1}{H} \sum_{h=1}^H a_{t,s}^{l,h} \right),
\]
where \( a_{t,s}^{l,h} \) denotes the attention weight from token \( t \) to token \( s \) at head \( h \) in layer \( l \), \( H \) is the number of heads per layer, \( \mathcal{L} \) is the set of selected layers for aggregation, and the constant $\eta$ defines late steps.

We then identify the top-$K$ most attended earlier steps based on \( \bar{a}_{k \to j} \):
$\mathcal{T}_k = \text{TopK}\left( \{ \bar{a}_{k \to j} \}_{j=1}^{k-1}, K \right),$
where \( \mathcal{T}_k \) is the set of indices corresponding to the top-attended steps. The step-level attention score for \( c_k \) is then defined as the proportion of these steps whose Reasoning Scores fall outside the normal range, either in the lower quartile or exceeding a high threshold $\tau$:
\[
\text{AttnScore}(c_k) = \frac{1}{K} \sum_{j \in \mathcal{T}_k} \mathds{1}_{\left( R^{j}_{\text{score}} \le \text{Quantile}_{1/4}(\mathcal{R}_{\text{score}}) \;\; \text{or} \;\; R^{j}_{\text{score}} \ge \tau \right)},
\]
where \( \mathds{1}_{(\cdot)} \) is the indicator function, \( \text{Quantile}_{1/4}(\mathcal{R}_{\text{score}}) \) denotes the first quartile of the reasoning scores (i.e., potentially shallow pattern-matching steps), and \( \tau \) is a threshold identifying potentially overthinking steps.

The trace-level attention score is computed by averaging over all later steps:
\begin{equation}
    \text{AttnScore}(\mathcal{C}) = \frac{1}{|\mathcal{C}_{\text{later}}|} \sum_{c_k \in \mathcal{C}_{\text{later}}} \text{AttnScore}(c_k),
    \label{eq:attn_score}
\end{equation}
which reflects the extent to which later reasoning steps attend to earlier incorrect steps.

\paragraph{Results.}
As shown in \autoref{fig:emp_1_result}(b) and (c) and Appendix~\ref{app:more_result_reason_hall}, across all three domains, hallucinated reasoning traces consistently yield significantly higher CV scores and Attention scores than truthful traces. This confirms that hallucinated traces are more fluctuating in reasoning depth (Pattern \#1) and more likely to attend prior incorrect steps (Pattern \#2), demonstrating the generalizability of both patterns beyond the initial case study (Section~\ref{sec:case_study}). Detailed settings are shown in Appendix~\ref{app:more_result_reason_hall}.

\begin{figure}[t]
    \centering
    \begin{subfigure}{0.28\textwidth}
        \includegraphics[width=\linewidth]{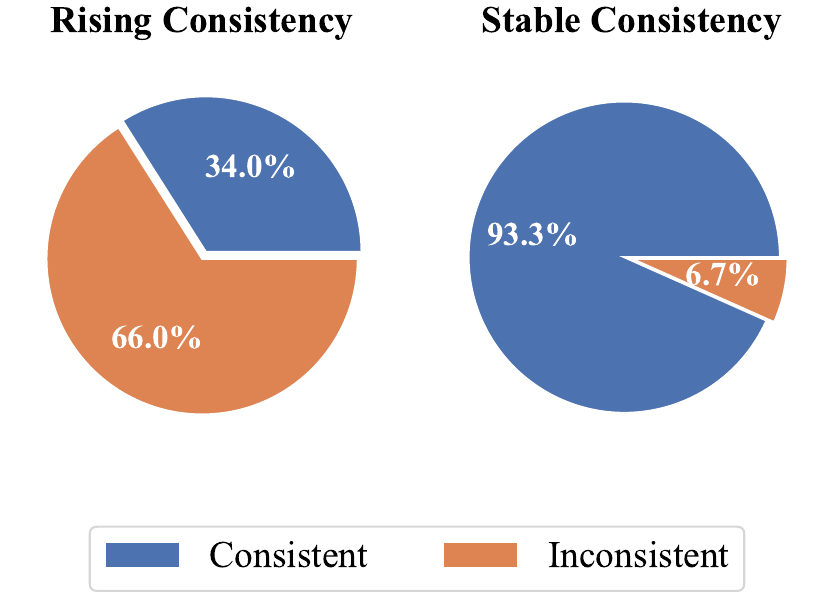}
        \caption*{(a)}
    \end{subfigure}
    \hfill
        \begin{subfigure}{0.2\textwidth}
        \includegraphics[width=\linewidth]{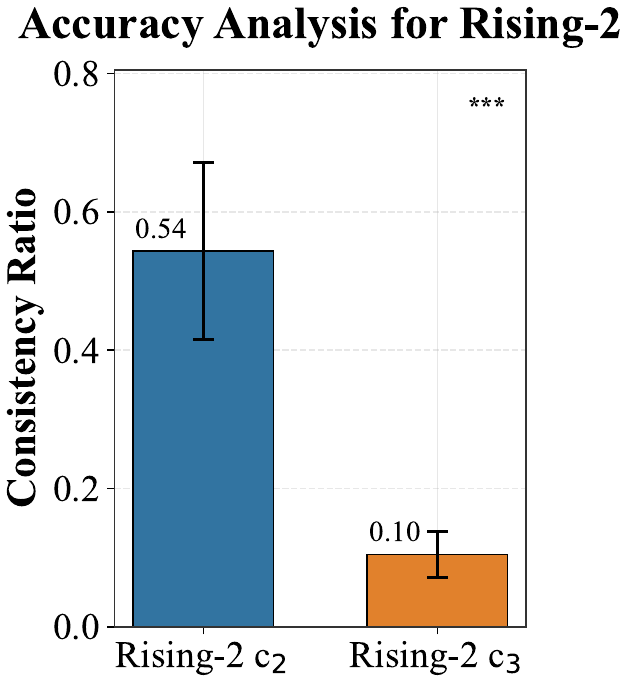}
        \caption*{(b)}
    \end{subfigure}
        \hfill
    \begin{subfigure}{0.28\textwidth}
        \includegraphics[width=\linewidth]{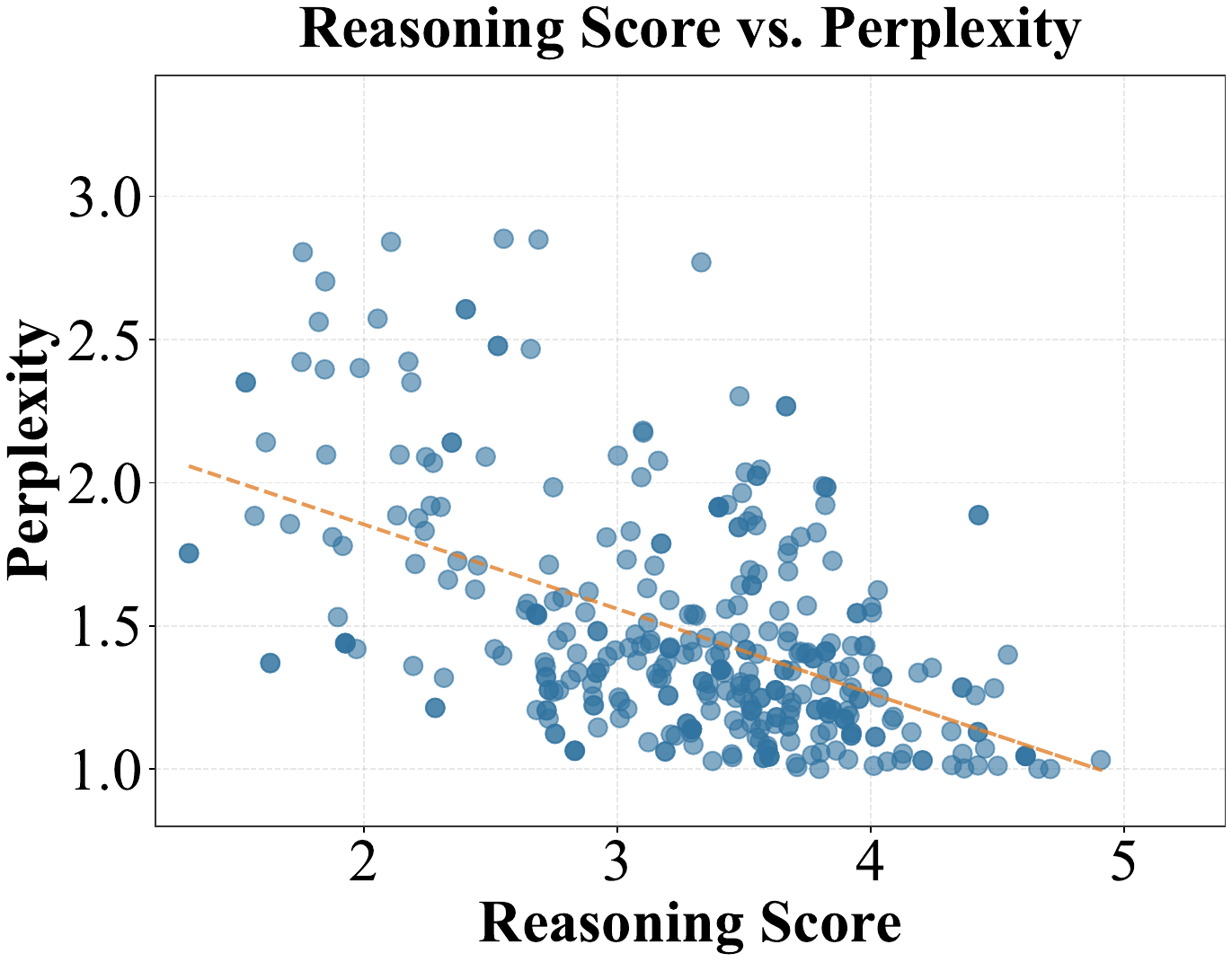}
        \caption*{(c)}
    \end{subfigure}
    \hfill
    \begin{subfigure}{0.18\textwidth}
        \includegraphics[width=\linewidth]{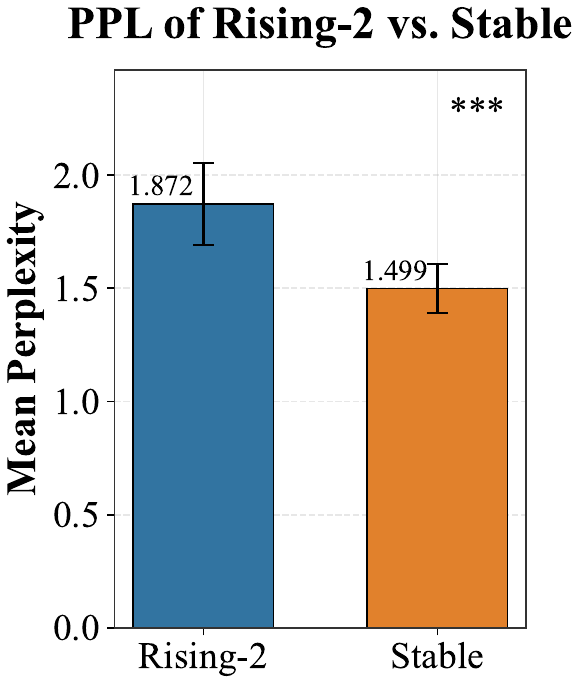}
        \caption*{(d)}
    \end{subfigure}
\caption{Analysis of Pattern~\#1: (a) Consistency Analysis (\textbf{Q1}); (b) Accuracy Comparison in Rising-2 triples (\textbf{Q2}); (c)  Reasoning score vs. perplexity and (d) \mbox{Perplexity of Rising-2 vs. Stable (\textbf{Q3}).}}
    \label{fig:emp_2}
\end{figure}


\subsection{Analyzing the Mechanisms Behind Reasoning Fluctuation}
\label{sec:under_pattern1}
We investigate the underlying mechanism behind \textbf{Pattern~\#1}, where hallucinated reasoning traces exhibit large fluctuations in reasoning depth. Building on our case study in Section~\ref{sec:case_study}, we hypothesize this stems from a built-in self-verification mechanism. Key questions still include: \textbf{Q1}: What triggers verification behavior in LRMs? \textbf{Q2}: Do excessively high reasoning scores reliably signal overthinking? \textbf{Q3}: If Q2 holds, what factors lead to the emergence of such overthinking steps?

To explore these, we construct step triples \((c_1, c_2, c_3)\) from reasoning traces: (1) \textbf{Stable} triples with minimal score variation from truthful traces; (2) \textbf{Rising-1} triples from hallucinated traces with a moderate score spike (\(R_{\text{score}}(c_3)<4\)), potentially triggered by shallow pattern-matching in $c_2$; and (3) \textbf{Rising-2} triples with extreme score spikes (\(R_{\text{score}}(c_3)>4\)), to probe overthinking behaviors.

\paragraph{Analysis.}
For \textbf{Q1}, we compare the logical consistency between \( c_1 \) and \( c_2 \) in Rising vs. Stable triples using GPT-4o judgments. As shown in \autoref{fig:emp_2}(a), stable triples show significantly higher consistency, suggesting that verification is more likely to be triggered when earlier steps are inconsistent.

Regarding \textbf{Q2}, we assess the accuracy of \( c_2 \) and \( c_3 \) in Rising-2 triples. \autoref{fig:emp_2}(b) shows that while \( c_2 \) is often correct, \( c_3 \) introduces errors, confirming that excessively high reasoning scores reliably signal overthinking. Prompts of \textbf{Q1} and \textbf{Q2} are shown in Appendix~\ref{app:prompt_hall_pattern}. 

To investigate \textbf{Q3}, we firstly analyze the correlation between reasoning depth and perplexity. As shown in \autoref{fig:emp_2}(c), reasoning steps with higher \( R_{\text{score}} \) generally exhibit lower perplexity, indicating more certainty outputs. However, \autoref{fig:emp_2}(d) reveals that in Rising-2 triples, \( c_3 \) steps, despite higher reasoning scores, have higher perplexity than those in stable triples, suggesting that overthinking may produce internally unstable generations.
We term this phenomenon \textit{spurious verification}, where the model performs misguided validation driven by outcome-based reward optimization.  This insight leads us to identify a new hallucination pattern:
\textbf{Pattern \#3:} Overthinking steps exhibit a positive correlation between \( R_{\text{score}} \) and perplexity. More details are provided in Appendix~\ref{app:detail_pattern_analysis}.

\section{Methods}
\label{sec:method}
\subsection{Reasoning Hallucination Detection}
\label{sec:rhd_method}
Building upon the patterns uncovered in our empirical study, we propose the \textbf{R}easoning \textbf{H}allucination \textbf{D}etection algorithm~(\textbf{RHD}). Our approach leverages the step-level Reasoning Score \( R_{\text{score}} \) to quantify thinking depth throughout the reasoning trace, and incorporates three identified indicators of hallucination: (1) Pattern~\#1: large fluctuations in reasoning scores during early steps, (2) Pattern~\#2: incorrect backtracking to earlier shallow or overthinking steps in later stages, and (3) Pattern~\#3: overthinking behavior where \( R_{\text{score}} \) and perplexity exhibit a positive correlation.

Given a question \( Q \) and its reasoning trace \( \mathcal{C} \) with step-level scores \( \mathcal{R}_{\text{score}} \), we define the overall Reasoning Hallucination Score as:
\begin{equation}
\mathcal{H}_{\mathcal{C}} = 
\underbrace{\alpha_1 \cdot \operatorname{Avg}(\mathcal{R}_{\text{score}})}_{\text{Overall Reasoning Depth}} 
+ \underbrace{\alpha_2 \cdot \operatorname{CV}(\mathcal{C})}_{\text{Pattern \#1}} 
+ \underbrace{\alpha_3 \cdot \operatorname{AttnScore}(\mathcal{C})}_{\text{Pattern \#2}} 
+ \underbrace{\alpha_4 \cdot \operatorname{PCC}(\mathcal{R}_{\text{score}}, \operatorname{PPL}(\mathcal{C}))}_{\text{Pattern \#3}},
\label{eq:hallucination_score}
\end{equation}
where \( \alpha_1, \alpha_2, \alpha_3, \alpha_4 \ge 0 \) are regression coefficients. \( \operatorname{Avg} \) denotes the average reasoning score, \( \operatorname{CV} \) (Eq.~\ref{eq:cv_score}) measures fluctuations during early-steps, \( \operatorname{AttnScore} \) (Eq.~\ref{eq:attn_score}) captures attention on earlier hallucinated steps, and \( \operatorname{PCC} \) refers to the Pearson correlation coefficient between reasoning scores and step-level perplexity \( \operatorname{PPL}(\mathcal{C}) \), computed according to Eq.~\ref{eq:ppl_trace}.

\subsection{Mitigating Reasoning Hallucinations via Step-Level Reasoning Score Shaping}
\label{sec:mitigation}

Reasoning hallucinations often stem from two types of flawed steps: (1) shallow pattern-matching, reflecting shortcut behaviors, and (2) overthinking, induced by excessive and misguided verification. A core factor is outcome-based RL, which only rewards the final answer and neglects intermediate steps~\citep{chen2025reasoning,valmeekam2024llms,transluce2024o3}, encouraging reward-hacking heuristics that may propagate through distillation~\citep{wang2025larger}.

To address this, we introduce an auxiliary process-level reward based on the \textbf{reasoning score} \( R_{\text{score}} \) from Section~\ref{sec:reasoning_score}, which measures the reasoning depth at each step. This encourages meaningful reasoning while penalizing shallow or overthinking steps.
We model the reasoning process as a finite-horizon MDP \( (\mathcal{S}, \mathcal{A}, P, r, \gamma) \), where \( s_t \in \mathcal{S} \) is the reasoning state at step \( t \), \( a_t \in \mathcal{A} \) denotes the next reasoning step, $P$ is the
transition probability and \( r_t \) is the reward:
\[
r_t =
\begin{cases}
0, & t < T, \\
R_{\rm final}, & t = T.
\end{cases}
\]

\paragraph{Reward Shaping with Reasoning Score.}
We apply potential-based reward shaping~\citep{ng1999policy}:
\[
\bar{r}_t = r_t + \gamma \Phi(s_{t+1}) - \Phi(s_t),
\quad \text{with } \Phi(s_T) = 0,
\]
which preserves the optimal policy while redistributing credit:
$
V'(s_t) = V(s_t) - \Phi(s_t),
$ where \( V(s_t) = \mathbb{E}_\pi \left[ \sum_{k=t}^T \gamma^{k-t} r_k \,\middle|\, s_t \right] \) is the value function of original reward and \( V'(s_t) \) is the shaped.

\paragraph{Potential Function Design.}
To avoid encouraging overthinking, we clip the reasoning score:
\[
\tilde R_{\rm score}(s_t) =
\begin{cases}
\alpha \cdot R_{\rm score}(s_t), & R_{\rm score}(s_t) \leq \tau, \\
0, & \text{otherwise},
\end{cases}
\quad
\Phi(s_t) = -\tilde R_{\rm score}(s_t),
\]
where \( \alpha > 0 \) and \( \tau \) control the weighting strength and the threshold for overthinking, respectively.

To understand the generalization benefit of our proposed reasoning score–based shaping, we derive a uniform convergence bound under augmented rewards:

\begin{theorem}[Generalization Gap with Augmented Rewards]
\label{thm:LRM:gb}
Let the policy class $\Pi$ be such that for any $\pi \in \Pi$, the augmented return $R(\pi,\xi) = \sum_{t=1}^{T} \gamma^{t-1} \bar{r}_t(\xi)$
is uniformly bounded in $[0, \bar{R}_{\max}]$ for any trajectory $\xi$ sampled from the environment.
Each trajectory $\xi = (s_1, a_1, \bar{r}_1, \dots, s_T, a_T, \bar{r}_T)$ denotes a complete multi-step reasoning trace.
Suppose that $\Pi$ has Rademacher complexity $\mathcal{R}_n(\Pi)$ based on $n$ independent training samples $\{\xi_i\}_{i=1}^n$. Then, with probability at least $1-\delta$, for any $\pi \in \Pi$ the following holds:
\[
  J_{\text{test}}(\pi) - J_{\text{train}}(\pi) \leq 2 \bar{R}_{\max}\, \mathcal{R}_n(\Pi) + \bar{R}_{\max} \sqrt{\frac{\log(1/\delta)}{2n}},
\]
where
$J_{\text{test}}(\pi)=\mathbb{E}_{\xi}[R(\pi,\xi)]$
is the expected test return and
$J_{\text{train}}(\pi)=\frac{1}{n}\sum_{i=1}^{n} R(\pi,\xi_i)$
is the empirical training return. 
\end{theorem}

The proof is given in Appendix~\ref{app:proof_gen}. Intuitively, our reasoning score acts as a regularizer that encourages logically consistent behaviors and effectively reduces the Rademacher complexity $\mathcal{R}_n(\Pi)$, thereby tightening the bound and improving generalization to unseen reasoning tasks.

\paragraph{Integrate into GRPO.}
To demonstrate compatibility with standard RL algorithms, we integrate the reasoning score shaping framework into the Group Relative Policy Optimization (GRPO), a scalable and widely used RL algorithm for reasoning model training~\cite{DeepSeekAI2025,shao2024deepseekmath}, yielding \textbf{GRPO-R}. All implementation and formulation details of GRPO-R are provided in Appendix~\ref{app:grpo_r_details}.

\begin{table}[t]
\centering
\resizebox{0.99\linewidth}{!}{
\renewcommand\arraystretch{1.15}
\setlength{\tabcolsep}{1.5mm}
\begin{tabular}{lllccccccccccccccc}
\toprule
\multirow{2}{*}{\textbf{LRMs}} & \multirow{2}{*}{\textbf{Categories}} & \multirow{2}{*}{\textbf{Methods}} & \multicolumn{5}{c}{\textbf{ReTruthQA (MATH)}} & \multicolumn{5}{c}{\textbf{ReTruthQA (Science)}} & \multicolumn{5}{c}{\textbf{ReTruthQA (MultiHopQA)}} \\
\cmidrule(lr){4-8} \cmidrule(lr){9-13} \cmidrule(lr){14-18}
~ & ~ & ~ & \textbf{AUC} & \multicolumn{1}{c|}{\textbf{PCC}} & \textbf{MC1} & \textbf{MC2} & \textbf{MC3} & \textbf{AUC} & \multicolumn{1}{c|}{\textbf{PCC}} & \textbf{MC1} & \textbf{MC2} & \textbf{MC3} & \textbf{AUC} & \multicolumn{1}{c|}{\textbf{PCC}} & \textbf{MC1} & \textbf{MC2} & \textbf{MC3} \\
\midrule
\multirow{14}{*}{\textbf{R1-7B}} & \multirow{3}{*}{\bf Ensemble} & ChainPoll~\cite{friel2023chainpoll} & 0.6384 & 0.2603 & 0.3020 & 0.2952 & 0.3583 & 0.6468 & 0.2612 & 0.2700 & 0.2580 & 0.3098 & 0.6297 & 0.2233 & 0.4208 & 0.3019 & 0.3954 \\
~ & ~ & LMvLM~\cite{cohen2023lm} & 0.6364 & 0.3728 & 0.3204 & 0.2504 & 0.3402 & 0.5345 & 0.1890 & 0.2600 & 0.2100 & 0.3113 & 0.6331 & 0.2759 & 0.3649 & 0.3049 & 0.3984 \\
~ & ~ & SelfCheckGPT~\cite{manakul2023selfcheckgpt} & 0.7727 & \underline{0.4598} & 0.4091 & 0.2784 & 0.4119 & \underline{0.6819} & \bf 0.2669 & 0.3793 & 0.3655 & 0.5320 & 0.6886 & \underline{0.2955} & 0.2553 & 0.1915 & 0.3118 \\
\cmidrule(lr){2-18}
~ & \multirow{3}{*}{\bf Uncertainty} & P(True)~\citep{kadavath2022language} & 0.7216 & 0.2681 & 0.5455 & 0.4068 & 0.5182 & 0.6207 & 0.2572 & 0.5172 & 0.4276 & 0.5533 & 0.5400 & 0.1684 & 0.4026 & 0.3030 & 0.4032 \\
~ & ~ & LN-Entropy~\citep{ren2022out} & 0.6896 & 0.3099 & 0.5000 & 0.3917 & 0.5096 & 0.5553 & 0.1129 & 0.3700 & 0.3200 & 0.4329 & 0.6123 & 0.2149 & 0.4156 & 0.3208 & 0.4461 \\
~ & ~ & PPL~\citep{malinin2020uncertainty} & 0.7025 & 0.2856 & \underline{0.5909} & \underline{0.4205} & 0.5267 & 0.5434 & 0.1144 & 0.3793 & 0.3034 & 0.3990 & 0.6432 & 0.2249 & 0.5745 & 0.4532 & 0.5241 \\
\cmidrule(lr){2-18}
~ & \bf Length & Length-Score~\citep{zeng2025revisiting} & 0.5351 & 0.0922 & 0.4318 & 0.2568 & 0.3408 & 0.5510 & 0.0911 & \underline{0.5793} & \underline{0.5034} & \underline{0.5737} & 0.5815 & 0.1496 & 0.5106 & 0.3887 & 0.4674 \\
\cmidrule(lr){2-18}
~ & \multirow{2}{*}{\bf PRM} & Qwen2.5-PRM800K~\citep{processbench} & 0.6601 & 0.2746 & 0.4773 & 0.3000 & 0.4572 & 0.6153 & 0.2203 & 0.4400 & 0.3605 & 0.4444 & 0.5694 & 0.1074 & 0.5065 & 0.4167 & 0.4990 \\
~ & ~ & Qwen2.5-PRM-7B~\citep{prmlessons} & 0.5563 & 0.1354 & 0.4318 & 0.2701 & 0.3913 & 0.5690 & 0.1275 & 0.2200 & 0.1425 & 0.2382 & 0.5422 & 0.0866 & 0.4026 & 0.2952 & 0.3947 
\\ \cmidrule(lr){2-18}
~ & \multirow{2}{*}{\bf LCM} & GPT4-o~\citep{achiam2023gpt} & 0.7513 & 0.3794 & 0.4091 & 0.2705 & 0.4131 & \bf 0.7045 & 0.2026 & 0.2500 & 0.2965 & 0.3200 & \underline{0.7123} & 0.2204 & 0.4043 & 0.2830 & 0.3704 \\
~ & ~ & Qwen2.5-32B~\citep{yang2024qwen2} & 0.6942 & 0.2082 & 0.2500 & 0.1955 & 0.2935 & 0.6525 & 0.2635 & 0.3103 & 0.2897 & 0.4458 & 0.6424 & 0.2056 & 0.4400 & 0.3300 & 0.4187 \\
\cmidrule(lr){2-18}
~ & \multirow{2}{*}{\bf Self-Aware} & UQAC~\citep{li2025languagemodeluncertaintyquantification} & 0.6671 & 0.2902 & 0.5833 & 0.3715 & \underline{0.5298} & 0.6303 & 0.2369 & 0.4700 & 0.3925 & 0.4885 & 0.6736 & 0.2583 & \underline{0.6623} & \underline{0.5335} & \underline{0.6425} \\
~ & ~ & EigenScore~\citep{cheninside} & \ \underline{0.7539} & 0.3868 & 0.4583 & 0.3250 & 0.3007 & 0.6488 & 0.2601 & 0.4260 & 0.3777 & 0.3815 & 0.6696 & 0.2858 & 0.5195 & 0.4113 & 0.3885 \\
\cmidrule(lr){2-18}
~ & \bf Ours & \bf RHD & \bf 0.7978 & \bf 0.4852 & \bf 0.6591 & \bf 0.4765 & \bf 0.5699 &  0.6528 & \underline{0.2662} & \bf 0.6207 & \bf 0.5448 & \bf 0.6009 & \bf 0.7361 & \bf 0.3863 & \bf 0.7660 & \bf 0.6255 & \bf 0.7103 \\
\midrule \midrule
\multirow{14}{*}{\textbf{R1-14B}} & \multirow{3}{*}{\bf Ensemble} & ChainPoll~\cite{friel2023chainpoll} & 0.5858 & 0.1658 & 0.2704 & 0.2535 & 0.3394 & 0.6640 & 0.3134 & 0.3261 & 0.1775 & 0.2188 & 0.5846 & 0.1607 & 0.2319 & 0.1972 & 0.2638 \\
~ & ~ & LMvLM~\cite{cohen2023lm} & 0.6620 & \bf 0.3835 & 0.2563 & 0.2507 & 0.3133 & 0.5435 & 0.2132 & 0.3333 & 0.2300 & 0.3421 & 0.6250 & 0.2914 & 0.2042 & 0.1885 & 0.2506 \\
~ & ~ & SelfCheckGPT~\cite{manakul2023selfcheckgpt} & 0.5714 & 0.2774 & 0.2462 & 0.2167 & 0.2930 & 0.5109 & 0.1048 & 0.3287 & 0.2566 & 0.3683 & 0.5208 & 0.1268 & 0.3167 & 0.3083 & 0.0320 \\
\cmidrule(lr){2-18}
~ & \multirow{3}{*}{\bf Uncertainty} & P(True)~\citep{kadavath2022language} & 0.6460 & 0.1443 & 0.2615 & 0.2374 & 0.4570 & 0.6645 & 0.2582 & 0.4828 & 0.3460 & 0.4885 & 0.6090 & 0.2057 & 0.3147 & 0.2508 & 0.4107 \\
~ & ~ & LN-Entropy~\citep{ren2022out} & 0.6423 & 0.2242 & 0.3479 & \underline{0.2939} & \underline{0.4754} & 0.6248 & 0.2134 & 0.5862 & 0.4147 & 0.5264 & 0.5337 & 0.0494 & 0.3125 & 0.2340 & 0.3678 \\
~ & ~ & PPL~\citep{malinin2020uncertainty} & 0.6526 & 0.2330 & \bf 0.3846 & 0.2744 & 0.4444 & 0.6219 & 0.1182 & 0.6000 & \underline{0.4215} & 0.5162 & 0.5337 & 0.1701 & 0.3058 & 0.2521 & 0.3630 \\
\cmidrule(lr){2-18}
~ & \bf Length & Length-Score~\citep{zeng2025revisiting} & 0.5184 & 0.0810 & 0.2817 & 0.2329 & 0.3400 & 0.5814 & 0.1487 & 0.5345 & 0.3848 & 0.4211 & 0.5971 & 0.1843 & 0.4711 & 0.3434 & 0.4284 \\
\cmidrule(lr){2-18}
~ & \multirow{2}{*}{\bf PRM} & Qwen2.5-PRM800K~\citep{processbench} & 0.5708 & 0.1285 & 0.3077 & 0.2697 & 0.4028 & \underline{0.7267} & \underline{0.4100} & 0.5862 & 0.3819 & 0.5132 & 0.6579 & 0.2451 & 0.4476 & 0.3366 & 0.4702 \\
~ & ~ & Qwen2.5-PRM-7B~\citep{prmlessons} & 0.5416 & 0.1249 & 0.3538 & 0.2918 & 0.4429 & 0.6983 & 0.3633 & 0.6133 & 0.4556 & \underline{0.5449} & 0.6674 & 0.2758 & 0.5045 & 0.3642 & 0.4853 \\
\cmidrule(lr){2-18}
~ & \multirow{2}{*}{\bf LCM} & GPT4-o~\citep{achiam2023gpt} & 0.6604 & 0.2458 & 0.2154 & 0.1785 & 0.3073 & 0.6265 & 0.1344 & 0.3333 & 0.1628 & 0.1933 & 0.6328 & 0.2356 & 0.2517 & 0.1878 & 0.2683 \\
~ & ~ & Qwen2.5-32B~\citep{yang2024qwen2} & 0.6650 & 0.3055 & 0.2676 & 0.2451 & 0.3632 & 0.6974 & 0.2381 & 0.3833 & 0.2150 & 0.3428 & \underline{0.7071} & 0.2716 & 0.3472 & 0.2517 & 0.4177 \\
\cmidrule(lr){2-18}
~ & \multirow{2}{*}{\bf Self-Aware} & UQAC~\citep{li2025languagemodeluncertaintyquantification} & 0.6374 & 0.2303 & 0.3444 & 0.2836 & \bf 0.5104 & {0.7157} & 0.3732 & \underline{0.6207} & 0.4170 & 0.5050 & 0.6952 & \underline{0.3397} & \underline{0.5417} & \underline{0.4222} & \underline{0.4988} \\
~ & ~ & EigenScore~\citep{cheninside} & \underline{0.6706} & \underline{0.3496} & 0.3282 & 0.2282 & 0.3388 & 0.6146 & 0.2228 & 0.4469 & 0.3508 & 0.3337 & 0.6719 & 0.3056 & 0.3694 & 0.3542 & 0.3750 \\
\cmidrule(lr){2-18}
~ & \bf Ours & \bf RHD & \bf{0.7292} & 0.3476 & \underline{0.3692} & \bf 0.3005 & 0.4644 & \bf 0.7649 & \bf 0.4506 & \bf 0.6667 & \bf 0.4714 & \bf 0.5671 & \bf 0.7255 & \bf 0.3742 & \bf 0.5785 & \bf 0.4421 & \bf 0.5154 \\
\bottomrule
\end{tabular}
}
\caption{Performance comparisons between RHD and baselines for Reasoning Hallucination Detection. The boldface represents the
best performance, and the underline represents the second-best.}
\label{tab:hallucination-eval}
\end{table}
\section{Experiments}
\label{sec:exp}
%

\subsection{Reasoning Hallucination Detection}
\label{sec:rhd_dataset}
\textbf{Data and Evaluation.}
We evaluate our RHD method on the \textbf{ReTruthQA} dataset spanning three reasoning domains: \texttt{Math}, \texttt{Science}, and \texttt{MultiHopQA} (construction details in Appendix~\ref{app:retruthqa}). We adopt two evaluation settings: 
(1) \textbf{Binary Detection}, which assesses the model’s ability to detect hallucinations in individual $(Q, C)$ pairs using AUC and PCC;
(2) \textbf{Multi-Trace Ranking}, which evaluates whether the model can rank truthful traces higher among multiple candidates $(Q, \{C_1, ..., C_N\})$, following TruthfulQA-MC~\citep{lin2021truthfulqa}. We report MC1, MC2, and MC3 to measure hallucination ranking accuracy (Evaluation details are in Appendix~\ref{app:retruthqa_eval}).

\textbf{Models and Baselines.}
We conduct experiments on two open-source LRM: \texttt{DeepSeek-R1-Distill-Qwen-7B} (R1-7B) and \texttt{DeepSeek-R1-Distill-Qwen-14B (R1-14B)} ~\cite{DeepSeekAI2025}. 
We compare our method against six categories of hallucination detection baselines:  
(1) \text{Ensemble based self-evaluation} (e.g., ChainPoll~\citep{friel2023chainpoll});  
(2) \text{Uncertainty based methods} (e.g., P(True)~\cite{kadavath2022language});  
(3) \text{Self-Awareness based approaches} (e.g., UQAC~\cite{li2025languagemodeluncertaintyquantification});  
(4) \text{LLM-as-Critic }(LCM) models (e.g., \texttt{GPT-4o});  
(5) \text{Process Reward Models }(PRMs) with step-level supervision (e.g., \texttt{Qwen2.5-Math-PRM});  
(6) \text{Length-based scoring}, which uses trace length as a proxy for hallucination likelihood. 
Baselines and RHD implementation details are provided in Appendix~\ref{app:retruthqa_eval} and ~\ref{app:rhd_exp}.

\textbf{Main Results.}
As shown in Table~\ref{tab:hallucination-eval}, RHD consistently outperforms most baselines across all ReTruthQA domains, model backbones, and evaluation settings, demonstrating strong robustness. Ensemble and LCM methods perform well in binary detection but struggle in multi-trace ranking, indicating difficulty in fine-grained comparison. Uncertainty-based methods are sensitive to output length, while Process Reward Models often suffer from limited generalization. In contrast, RHD directly leverages reasoning mechanisms for more accurate detection. Self-awareness methods perform competitively but lack explicit reasoning analysis. Interestingly, the Length-based baseline performs well in multi-trace settings—supporting the intuition that overly long traces are more error-prone, but underperforms in binary detection, limiting its generality.
These findings highlight the effectiveness of RHD modeling internal reasoning patterns for hallucination detection. Additional ablations and sensitivity studies are provided in Appendix~\ref{app:ablation_rhd} and~\ref{app:sen_rhd}.

\begin{table}[t]
\centering
\resizebox{0.99\linewidth}{!}{
\renewcommand\arraystretch{1.15}
\setlength{\tabcolsep}{1.5mm}
\begin{tabular}{lcccccccccc}
\toprule
\textbf{Models} & \multicolumn{5}{c}{\textbf{DeepSeek-R1-1.5B}} & \multicolumn{5}{c}{\textbf{Qwen2.5-1.5B-Instruct}} \\
\cmidrule(lr){2-6} \cmidrule(lr){7-11}
~ & \textbf{MATH500} & \textbf{AIME(2024)} & \textbf{GPQA(diamond)} & \textbf{GPQA(main)} & \textbf{GPQA(extended)} 
  & \textbf{MATH500} & \textbf{AIME(2024)} & \textbf{GPQA(diamond)} & \textbf{GPQA(main)} & \textbf{GPQA(extended)} \\
\midrule
Base       & 0.772 & 0.333 & 0.354 & 0.333 & 0.339 & 0.466 & 0.100 & 0.202 & 0.197 & 0.211 \\
+GRPO    & 0.770 & 0.333 & 0.359 & 0.335 & \textbf{0.359} & 0.480 & 0.033 & \textbf{0.247} & 0.214 & 0.266 \\
+GRPO-R  & \textbf{0.788} & \textbf{0.367} & \textbf{0.414} & \textbf{0.371} & {0.357} 
                       & \textbf{0.490} & \textbf{0.133} & \textbf{0.247} & \textbf{0.243} & \textbf{0.275} \\
\bottomrule
\end{tabular}
}
\caption{Performance comparisons between GRPO-R and baselines. Bold indicates the best result.}
\label{tab:hallucination_mitigation}
\end{table}

\subsection{Reasoning Hallucination Mitigation}

\textbf{Experimental Setting.} 
To assess the effectiveness of GRPO-R in reducing reasoning hallucinations, we fine-tune \texttt{Qwen2.5-1.5B-Instruct} and \texttt{DeepSeek-R1-1.5B} on 2,000 examples from \texttt{OpenR1-Math-220K}~\citep{openr1math2024} using either GRPO or our proposed GRPO-R. We evaluate the accuracy~\citep{math_verify} on two in-domain math benchmarks—\texttt{MATH500}~\cite{lightman2023let} and \texttt{AIME 2024}~\cite{aimo_validation_aime}—and an out-of-distribution science benchmark—\texttt{GPQA}~\citep{rein2024gpqa}. Implementation details are in Appendix~\ref{app:imp_rhm}.

\textbf{Main Results.}
As shown in Table~\ref{tab:hallucination_mitigation}, GRPO-R outperforms GRPO across most of the tasks, indicating that shaping reasoning steps via the reasoning score enhances both factual accuracy and reasoning reliability. Gains on \texttt{GPQA} further suggest improved generalization beyond training distribution. Additional sensitivity analyses are in Appendix~\ref{app:sen_rhm}. Hallucination mitigation experiments in data distillation in Appendix~\ref{app:rhd_distill} further validate the effectiveness of our proposed RHD model.

\section{Conclusion and Limitation}
\label{sec:limi}
We tackle the challenge of \textit{Reasoning Hallucination} in LRMs, where models produce logically coherent but factually incorrect reasoning traces. To address this, we propose the \textbf{Reasoning Score}, a step-level metric derived from the mechanistic interpretability of reasoning, which effectively quantifies reasoning depth. Based on this score, we identify three key hallucination patterns: early-stage depth fluctuations, incorrect backtracking and spurious verification-induced overthinking—and develop the \textbf{RHD} framework for hallucination detection. Building on these insights, we introduce \textbf{GRPO-R}, a reinforcement learning method that integrates step-level reasoning rewards via potential-based shaping, improving both accuracy and robustness across reasoning benchmarks.

\paragraph{Limitation.}
RHD relies on internal model activations and is thus limited to open-source LRMs with accessible activations. Its application to black-box models remains an open challenge. Nonetheless, the discovered patterns and metrics could inspire proxy-based extensions.  Additionally, experiments are conducted on moderate-scale models and datasets due to computational constraints; future work includes scaling up to broader domains and model families.

\bibliographystyle{abbrvnat}
\bibliography{neurips_2025}

\begin{thebibliography}{61}
\providecommand{\natexlab}[1]{#1}
\providecommand{\url}[1]{\texttt{#1}}
\expandafter\ifx\csname urlstyle\endcsname\relax
  \providecommand{\doi}[1]{doi: #1}\else
  \providecommand{\doi}{doi: \begingroup \urlstyle{rm}\Url}\fi

\bibitem[Achiam et~al.(2023)Achiam, Adler, Agarwal, Ahmad, Akkaya, Aleman, Almeida, Altenschmidt, Altman, Anadkat, et~al.]{achiam2023gpt}
J.~Achiam, S.~Adler, S.~Agarwal, L.~Ahmad, I.~Akkaya, F.~L. Aleman, D.~Almeida, J.~Altenschmidt, S.~Altman, S.~Anadkat, et~al.
\newblock Gpt-4 technical report.
\newblock \emph{arXiv preprint arXiv:2303.08774}, 2023.

\bibitem[AI-MO(2024{\natexlab{a}})]{aimo_validation_aime}
AI-MO.
\newblock Ai-mo/aimo-validation-aime.
\newblock \url{https://huggingface.co/datasets/AI-MO/aimo-validation-aime}, 2024{\natexlab{a}}.
\newblock Apache 2.0 License.

\bibitem[AI-MO(2024{\natexlab{b}})]{aimo_validation_amc}
AI-MO.
\newblock Ai-mo/aimo-validation-amc.
\newblock \url{https://huggingface.co/datasets/AI-MO/aimo-validation-amc}, 2024{\natexlab{b}}.
\newblock Apache 2.0 License.

\bibitem[Bruckmaier et~al.(2021)Bruckmaier, Krauss, Binder, Hilbert, and Brunner]{bruckmaier2021tversky}
G.~Bruckmaier, S.~Krauss, K.~Binder, S.~Hilbert, and M.~Brunner.
\newblock Tversky and kahneman’s cognitive illusions: who can solve them, and why?
\newblock \emph{Frontiers in psychology}, 12:\penalty0 584689, 2021.

\bibitem[Chen et~al.(2024)Chen, Liu, Chen, Gu, Wu, Tao, Fu, and Ye]{cheninside}
C.~Chen, K.~Liu, Z.~Chen, Y.~Gu, Y.~Wu, M.~Tao, Z.~Fu, and J.~Ye.
\newblock Inside: Llms' internal states retain the power of hallucination detection.
\newblock In \emph{The Twelfth International Conference on Learning Representations}, 2024.

\bibitem[Chen et~al.(2025{\natexlab{a}})Chen, Zhang, Zhu, Liu, Gao, Xiong, Li, and He]{chen2025bring}
S.~Chen, J.~Zhang, T.~Zhu, W.~Liu, S.~Gao, M.~Xiong, M.~Li, and J.~He.
\newblock Bring reason to vision: Understanding perception and reasoning through model merging.
\newblock \emph{arXiv preprint arXiv:2505.05464}, 2025{\natexlab{a}}.

\bibitem[Chen et~al.(2025{\natexlab{b}})Chen, Benton, Radhakrishnan, Uesato, Denison, Schulman, Somani, Hase, Wagner, Roger, et~al.]{chen2025reasoning}
Y.~Chen, J.~Benton, A.~Radhakrishnan, J.~Uesato, C.~Denison, J.~Schulman, A.~Somani, P.~Hase, M.~Wagner, F.~Roger, et~al.
\newblock Reasoning models don't always say what they think.
\newblock \emph{arXiv preprint arXiv:2505.05410}, 2025{\natexlab{b}}.

\bibitem[Cohen et~al.(2023)Cohen, Hamri, Geva, and Globerson]{cohen2023lm}
R.~Cohen, M.~Hamri, M.~Geva, and A.~Globerson.
\newblock Lm vs lm: Detecting factual errors via cross examination.
\newblock \emph{arXiv preprint arXiv:2305.13281}, 2023.

\bibitem[Elhage et~al.(2021)Elhage, Nanda, Olsson, Henighan, Joseph, Mann, Askell, Bai, Chen, Conerly, DasSarma, Drain, Ganguli, Hatfield-Dodds, Hernandez, Jones, Kernion, Lovitt, Ndousse, Amodei, Brown, Clark, Kaplan, McCandlish, and Olah]{transformercircuits2021}
N.~Elhage, N.~Nanda, C.~Olsson, T.~Henighan, N.~Joseph, B.~Mann, A.~Askell, Y.~Bai, A.~Chen, T.~Conerly, N.~DasSarma, D.~Drain, D.~Ganguli, Z.~Hatfield-Dodds, D.~Hernandez, A.~Jones, J.~Kernion, L.~Lovitt, K.~Ndousse, D.~Amodei, T.~Brown, J.~Clark, J.~Kaplan, S.~McCandlish, and C.~Olah.
\newblock A mathematical framework for transformer circuits.
\newblock Transformer Circuits Thread, 2021.
\newblock URL \url{https://transformer-circuits.pub/2021/framework/index.html}.

\bibitem[et~al.(2025)]{DeepSeekAI2025}
D.-A. et~al.
\newblock Deepseek-r1: Incentivizing reasoning capability in llms via reinforcement learning.
\newblock \emph{arXiv preprint arXiv:2501.12948}, 2025.
\newblock URL \url{https://arxiv.org/abs/2501.12948}.

\bibitem[Everitt(1998)]{everitt1998cambridge}
B.~Everitt.
\newblock The cambridge dictionary of statistics.
\newblock In \emph{The Cambridge dictionary of statistics}, pages 360--360. 1998.

\bibitem[Ferrando and Voita(2024)]{ferrando2024information}
J.~Ferrando and E.~Voita.
\newblock Information flow routes: Automatically interpreting language models at scale.
\newblock \emph{arXiv preprint arXiv:2403.00824}, 2024.

\bibitem[Ferrando et~al.(2024)Ferrando, Sarti, Bisazza, and Costa-juss{\`a}]{ferrando2024primer}
J.~Ferrando, G.~Sarti, A.~Bisazza, and M.~R. Costa-juss{\`a}.
\newblock A primer on the inner workings of transformer-based language models.
\newblock \emph{arXiv preprint arXiv:2405.00208}, 2024.

\bibitem[Friel and Sanyal(2023)]{friel2023chainpoll}
R.~Friel and A.~Sanyal.
\newblock Chainpoll: A high efficacy method for llm hallucination detection.
\newblock \emph{arXiv preprint arXiv:2310.18344}, 2023.

\bibitem[Geva et~al.(2021)Geva, Schuster, Berant, and Levy]{geva2021transformer}
M.~Geva, R.~Schuster, J.~Berant, and O.~Levy.
\newblock Transformer feed-forward layers are key-value memories.
\newblock In \emph{Proceedings of the 2021 Conference on Empirical Methods in Natural Language Processing}, pages 5484--5495, 2021.

\bibitem[Hanna et~al.(2024)Hanna, Liu, and Variengien]{hanna2024does}
M.~Hanna, O.~Liu, and A.~Variengien.
\newblock How does gpt-2 compute greater-than?: Interpreting mathematical abilities in a pre-trained language model.
\newblock \emph{Advances in Neural Information Processing Systems}, 36, 2024.

\bibitem[Hao et~al.(2024)Hao, Sukhbaatar, Su, Li, Hu, Weston, and Tian]{hao2024training}
S.~Hao, S.~Sukhbaatar, D.~Su, X.~Li, Z.~Hu, J.~Weston, and Y.~Tian.
\newblock Training large language models to reason in a continuous latent space.
\newblock \emph{arXiv preprint arXiv:2412.06769}, 2024.

\bibitem[He et~al.(2025)He, Li, Liu, Wang, Bu, Zhang, Peng, Zhang, Zheng, Su, and Zheng]{he2025largelanguagemodelsdetect}
Y.~He, S.~Li, J.~Liu, W.~Wang, X.~Bu, G.~Zhang, Z.~Peng, Z.~Zhang, Z.~Zheng, W.~Su, and B.~Zheng.
\newblock Can large language models detect errors in long chain-of-thought reasoning?, 2025.
\newblock URL \url{https://arxiv.org/abs/2502.19361}.

\bibitem[Ho et~al.(2020)Ho, Nguyen, Sugawara, and Aizawa]{ho2020constructing}
X.~Ho, A.-K.~D. Nguyen, S.~Sugawara, and A.~Aizawa.
\newblock Constructing a multi-hop qa dataset for comprehensive evaluation of reasoning steps.
\newblock \emph{arXiv preprint arXiv:2011.01060}, 2020.

\bibitem[{Hugging Face}(2025)]{math_verify}
{Hugging Face}.
\newblock Math-verify: A rule-based mathematical answer verification library, 2025.
\newblock URL \url{https://github.com/huggingface/Math-Verify}.

\bibitem[Kadavath et~al.(2022)Kadavath, Conerly, Askell, Henighan, Drain, Perez, Schiefer, Hatfield-Dodds, DasSarma, Tran-Johnson, et~al.]{kadavath2022language}
S.~Kadavath, T.~Conerly, A.~Askell, T.~Henighan, D.~Drain, E.~Perez, N.~Schiefer, Z.~Hatfield-Dodds, N.~DasSarma, E.~Tran-Johnson, et~al.
\newblock Language models (mostly) know what they know.
\newblock \emph{arXiv preprint arXiv:2207.05221}, 2022.

\bibitem[Kahneman(2011)]{kahneman2011thinking}
D.~Kahneman.
\newblock \emph{Thinking, fast and slow}.
\newblock macmillan, 2011.

\bibitem[Li et~al.(2025{\natexlab{a}})Li, Cao, Griggs, Liu, Mo, Tang, Hegde, Hakhamaneshi, Patil, Zaharia, et~al.]{li2025llms}
D.~Li, S.~Cao, T.~Griggs, S.~Liu, X.~Mo, E.~Tang, S.~Hegde, K.~Hakhamaneshi, S.~G. Patil, M.~Zaharia, et~al.
\newblock Llms can easily learn to reason from demonstrations structure, not content, is what matters!
\newblock \emph{arXiv preprint arXiv:2502.07374}, 2025{\natexlab{a}}.

\bibitem[Li et~al.(2024{\natexlab{a}})Li, Patel, Vi{\'e}gas, Pfister, and Wattenberg]{li2024inference}
K.~Li, O.~Patel, F.~Vi{\'e}gas, H.~Pfister, and M.~Wattenberg.
\newblock Inference-time intervention: Eliciting truthful answers from a language model.
\newblock \emph{Advances in Neural Information Processing Systems}, 36, 2024{\natexlab{a}}.

\bibitem[Li et~al.(2025{\natexlab{b}})Li, Qiang, Moukheiber, and Zhang]{li2025languagemodeluncertaintyquantification}
Y.~Li, R.~Qiang, L.~Moukheiber, and C.~Zhang.
\newblock Language model uncertainty quantification with attention chain, 2025{\natexlab{b}}.
\newblock URL \url{https://arxiv.org/abs/2503.19168}.

\bibitem[Li et~al.(2024{\natexlab{b}})Li, Jiang, Xie, Song, Lian, and Wei]{li2024understanding}
Z.~Li, G.~Jiang, H.~Xie, L.~Song, D.~Lian, and Y.~Wei.
\newblock Understanding and patching compositional reasoning in llms.
\newblock \emph{arXiv preprint arXiv:2402.14328}, 2024{\natexlab{b}}.

\bibitem[Lightman et~al.(2023)Lightman, Kosaraju, Burda, Edwards, Baker, Lee, Leike, Schulman, Sutskever, and Cobbe]{lightman2023let}
H.~Lightman, V.~Kosaraju, Y.~Burda, H.~Edwards, B.~Baker, T.~Lee, J.~Leike, J.~Schulman, I.~Sutskever, and K.~Cobbe.
\newblock Let's verify step by step.
\newblock In \emph{The Twelfth International Conference on Learning Representations}, 2023.

\bibitem[Lin et~al.(2021)Lin, Hilton, and Evans]{lin2021truthfulqa}
S.~Lin, J.~Hilton, and O.~Evans.
\newblock Truthfulqa: Measuring how models mimic human falsehoods.
\newblock \emph{arXiv preprint arXiv:2109.07958}, 2021.

\bibitem[Luo et~al.(2024)Luo, Liu, Liu, Phatale, Guo, Lara, Li, Shu, Zhu, Meng, et~al.]{luo2024improve}
L.~Luo, Y.~Liu, R.~Liu, S.~Phatale, M.~Guo, H.~Lara, Y.~Li, L.~Shu, Y.~Zhu, L.~Meng, et~al.
\newblock Improve mathematical reasoning in language models by automated process supervision.
\newblock \emph{arXiv preprint arXiv:2406.06592}, 2024.

\bibitem[Malinin and Gales(2020)]{malinin2020uncertainty}
A.~Malinin and M.~Gales.
\newblock Uncertainty estimation in autoregressive structured prediction.
\newblock \emph{arXiv preprint arXiv:2002.07650}, 2020.

\bibitem[Manakul et~al.(2023)Manakul, Liusie, and Gales]{manakul2023selfcheckgpt}
P.~Manakul, A.~Liusie, and M.~Gales.
\newblock Selfcheckgpt: Zero-resource black-box hallucination detection for generative large language models.
\newblock In \emph{Proceedings of the 2023 Conference on Empirical Methods in Natural Language Processing}, pages 9004--9017, 2023.

\bibitem[Mirzadeh et~al.(2024)Mirzadeh, Alizadeh, Shahrokhi, Tuzel, Bengio, and Farajtabar]{mirzadeh2024gsm}
I.~Mirzadeh, K.~Alizadeh, H.~Shahrokhi, O.~Tuzel, S.~Bengio, and M.~Farajtabar.
\newblock Gsm-symbolic: Understanding the limitations of mathematical reasoning in large language models.
\newblock \emph{arXiv preprint arXiv:2410.05229}, 2024.

\bibitem[Ng et~al.(1999)Ng, Harada, and Russell]{ng1999policy}
A.~Y. Ng, D.~Harada, and S.~Russell.
\newblock Policy invariance under reward transformations: Theory and application to reward shaping.
\newblock In \emph{Icml}, volume~99, pages 278--287. Citeseer, 1999.

\bibitem[Nikankin et~al.(2025)Nikankin, Reusch, Mueller, and Belinkov]{nikankin2025arithmetic}
Y.~Nikankin, A.~Reusch, A.~Mueller, and Y.~Belinkov.
\newblock Arithmetic without algorithms: Language models solve math with a bag of heuristics.
\newblock In \emph{The Thirteenth International Conference on Learning Representations}, 2025.
\newblock URL \url{https://openreview.net/forum?id=O9YTt26r2P}.

\bibitem[nostalgebraist(2020)]{nostalgebraist2020}
nostalgebraist.
\newblock Interpreting {GPT}: the logit lens.
\newblock \emph{AI Alignment Forum}, 2020.
\newblock URL \url{https://www.alignmentforum.org/posts/AcKRB8wDpdaN6v6ru/interpreting-gpt-the-logit-lens}.

\bibitem[{OpenAI}(2025)]{openai2025o3o4}
{OpenAI}.
\newblock Openai o3 and o4-mini system card, April 2025.
\newblock URL \url{https://openai.com/index/o3-o4-mini-system-card/}.

\bibitem[Prasad et~al.(2023)Prasad, Saha, Zhou, and Bansal]{prasad2023receval}
A.~Prasad, S.~Saha, X.~Zhou, and M.~Bansal.
\newblock Receval: Evaluating reasoning chains via correctness and informativeness.
\newblock \emph{arXiv preprint arXiv:2304.10703}, 2023.

\bibitem[Press et~al.(2022)Press, Zhang, Min, Schmidt, Smith, and Lewis]{press2022measuring}
O.~Press, M.~Zhang, S.~Min, L.~Schmidt, N.~A. Smith, and M.~Lewis.
\newblock Measuring and narrowing the compositionality gap in language models.
\newblock \emph{arXiv preprint arXiv:2210.03350}, 2022.

\bibitem[Rein et~al.(2024)Rein, Hou, Stickland, Petty, Pang, Dirani, Michael, and Bowman]{rein2024gpqa}
D.~Rein, B.~L. Hou, A.~C. Stickland, J.~Petty, R.~Y. Pang, J.~Dirani, J.~Michael, and S.~R. Bowman.
\newblock Gpqa: A graduate-level google-proof q\&a benchmark.
\newblock In \emph{First Conference on Language Modeling}, 2024.

\bibitem[Ren et~al.(2022)Ren, Luo, Zhao, Krishna, Saleh, Lakshminarayanan, and Liu]{ren2022out}
J.~Ren, J.~Luo, Y.~Zhao, K.~Krishna, M.~Saleh, B.~Lakshminarayanan, and P.~J. Liu.
\newblock Out-of-distribution detection and selective generation for conditional language models.
\newblock In \emph{The Eleventh International Conference on Learning Representations}, 2022.

\bibitem[Shao et~al.(2024)Shao, Wang, Zhu, Xu, Song, Bi, Zhang, Zhang, Li, Wu, et~al.]{shao2024deepseekmath}
Z.~Shao, P.~Wang, Q.~Zhu, R.~Xu, J.~Song, X.~Bi, H.~Zhang, M.~Zhang, Y.~Li, Y.~Wu, et~al.
\newblock Deepseekmath: Pushing the limits of mathematical reasoning in open language models.
\newblock \emph{arXiv preprint arXiv:2402.03300}, 2024.

\bibitem[Stolfo et~al.(2023)Stolfo, Belinkov, and Sachan]{stolfo2023a}
A.~Stolfo, Y.~Belinkov, and M.~Sachan.
\newblock A mechanistic interpretation of arithmetic reasoning in language models using causal mediation analysis.
\newblock In \emph{The 2023 Conference on Empirical Methods in Natural Language Processing}, 2023.
\newblock URL \url{https://openreview.net/forum?id=aB3Hwh4UzP}.

\bibitem[Team(2024)]{openr1math2024}
O.-R. Team.
\newblock Openr1-math-220k: A dataset for mathematical reasoning, 2024.
\newblock URL \url{https://huggingface.co/datasets/open-r1/OpenR1-Math-220k}.

\bibitem[Tentori et~al.(2004)Tentori, Bonini, and Osherson]{tentori2004conjunction}
K.~Tentori, N.~Bonini, and D.~Osherson.
\newblock The conjunction fallacy: A misunderstanding about conjunction?
\newblock \emph{Cognitive Science}, 28\penalty0 (3):\penalty0 467--477, 2004.

\bibitem[{Transluce Research}(2024)]{transluce2024o3}
{Transluce Research}.
\newblock Investigating truthfulness in a pre-release o3 model, 2024.
\newblock URL \url{https://transluce.org/investigating-o3-truthfulness}.

\bibitem[Trivedi et~al.(2022)Trivedi, Balasubramanian, Khot, and Sabharwal]{trivedi2022musique}
H.~Trivedi, N.~Balasubramanian, T.~Khot, and A.~Sabharwal.
\newblock Musique: Multihop questions via single-hop question composition.
\newblock \emph{Transactions of the Association for Computational Linguistics}, 10:\penalty0 539--554, 2022.

\bibitem[Valmeekam et~al.()Valmeekam, Stechly, and Kambhampati]{valmeekam2024llms}
K.~Valmeekam, K.~Stechly, and S.~Kambhampati.
\newblock Llms still can't plan; can lrms? a preliminary evaluation of openai's o1 on planbench.
\newblock In \emph{NeurIPS 2024 Workshop on Open-World Agents}.

\bibitem[{Vectara Research}(2025)]{vectara2025deepseek}
{Vectara Research}.
\newblock Deepseek-r1 hallucinates more than deepseek-v3, 2025.
\newblock URL \url{https://www.vectara.com/blog/deepseek-r1-hallucinates-more-than-deepseek-v3}.

\bibitem[Wang et~al.(2025)Wang, Tan, Jin, Wang, Panda, and Shen]{wang2025larger}
X.~Wang, S.~Tan, M.~Jin, W.~Y. Wang, R.~Panda, and Y.~Shen.
\newblock Do larger language models imply better reasoning? a pretraining scaling law for reasoning.
\newblock \emph{arXiv preprint arXiv:2504.03635}, 2025.

\bibitem[Weis and Kunde(2024)]{weis2024switching}
P.~P. Weis and W.~Kunde.
\newblock Switching between different cognitive strategies induces switch costs as evidenced by switches between manual and mental object rotation.
\newblock \emph{Scientific Reports}, 14\penalty0 (1):\penalty0 6217, 2024.

\bibitem[Wu et~al.(2024)Wu, Wang, Xiao, Peng, and Fu]{wu2024retrieval}
W.~Wu, Y.~Wang, G.~Xiao, H.~Peng, and Y.~Fu.
\newblock Retrieval head mechanistically explains long-context factuality.
\newblock \emph{arXiv preprint arXiv:2404.15574}, 2024.

\bibitem[Xu et~al.(2024)Xu, Diao, Yang, and Wang]{xu2024can}
X.~Xu, S.~Diao, C.~Yang, and Y.~Wang.
\newblock Can we verify step by step for incorrect answer detection?
\newblock \emph{arXiv preprint arXiv:2402.10528}, 2024.

\bibitem[Yan et~al.(2025)Yan, Xu, Du, Yao, Wang, Guo, and Chen]{yan2025recitation}
K.~Yan, Y.~Xu, Z.~Du, X.~Yao, Z.~Wang, X.~Guo, and J.~Chen.
\newblock Recitation over reasoning: How cutting-edge language models can fail on elementary school-level reasoning problems?
\newblock \emph{arXiv preprint arXiv:2504.00509}, 2025.

\bibitem[Yang et~al.(2024)Yang, Yang, Zhang, Hui, Zheng, Yu, Li, Liu, Huang, Wei, et~al.]{yang2024qwen2}
A.~Yang, B.~Yang, B.~Zhang, B.~Hui, B.~Zheng, B.~Yu, C.~Li, D.~Liu, F.~Huang, H.~Wei, et~al.
\newblock Qwen2. 5 technical report.
\newblock \emph{arXiv preprint arXiv:2412.15115}, 2024.

\bibitem[Yang et~al.(2018)Yang, Qi, Zhang, Bengio, Cohen, Salakhutdinov, and Manning]{yang2018hotpotqa}
Z.~Yang, P.~Qi, S.~Zhang, Y.~Bengio, W.~W. Cohen, R.~Salakhutdinov, and C.~D. Manning.
\newblock Hotpotqa: A dataset for diverse, explainable multi-hop question answering.
\newblock \emph{arXiv preprint arXiv:1809.09600}, 2018.

\bibitem[Yu et~al.(2023)Yu, Merullo, and Pavlick]{yu2023characterizing}
Q.~Yu, J.~Merullo, and E.~Pavlick.
\newblock Characterizing mechanisms for factual recall in language models.
\newblock In \emph{The 2023 Conference on Empirical Methods in Natural Language Processing}, 2023.

\bibitem[Zeng et~al.(2025)Zeng, Cheng, Yin, Zhou, and Qiu]{zeng2025revisiting}
Z.~Zeng, Q.~Cheng, Z.~Yin, Y.~Zhou, and X.~Qiu.
\newblock Revisiting the test-time scaling of o1-like models: Do they truly possess test-time scaling capabilities?
\newblock \emph{arXiv preprint arXiv:2502.12215}, 2025.

\bibitem[Zhang et~al.(2025)Zhang, Zheng, Wu, Zhang, Lin, Yu, Liu, Zhou, and Lin]{prmlessons}
Z.~Zhang, C.~Zheng, Y.~Wu, B.~Zhang, R.~Lin, B.~Yu, D.~Liu, J.~Zhou, and J.~Lin.
\newblock The lessons of developing process reward models in mathematical reasoning.
\newblock \emph{arXiv preprint arXiv:2501.07301}, 2025.

\bibitem[Zheng et~al.(2024{\natexlab{a}})Zheng, Zhang, Zhang, Lin, Lu, Yu, Liu, Zhou, and Lin]{processbench}
C.~Zheng, Z.~Zhang, B.~Zhang, R.~Lin, K.~Lu, B.~Yu, D.~Liu, J.~Zhou, and J.~Lin.
\newblock Processbench: Identifying process errors in mathematical reasoning.
\newblock \emph{arXiv preprint arXiv:2412.06559}, 2024{\natexlab{a}}.

\bibitem[Zheng et~al.(2024{\natexlab{b}})Zheng, Zhang, Zhang, Lin, Lu, Yu, Liu, Zhou, and Lin]{zheng2024processbenchidentifyingprocesserrors}
C.~Zheng, Z.~Zhang, B.~Zhang, R.~Lin, K.~Lu, B.~Yu, D.~Liu, J.~Zhou, and J.~Lin.
\newblock Processbench: Identifying process errors in mathematical reasoning, 2024{\natexlab{b}}.
\newblock URL \url{https://arxiv.org/abs/2412.06559}.

\bibitem[Zhou et~al.(2024)Zhou, Feng, Zhu, Qian, and Mao]{zhou2024unibias}
H.~Zhou, Z.~Feng, Z.~Zhu, J.~Qian, and K.~Mao.
\newblock Unibias: Unveiling and mitigating llm bias through internal attention and ffn manipulation.
\newblock \emph{arXiv preprint arXiv:2405.20612}, 2024.

\end{thebibliography}


\appendix

\section{Proof of Generalization Gap with Augmented Rewards}
\label{app:proof_gen}


\begin{proof}[Proof of Theorem~\ref{thm:LRM:gb}]
For any policy $\pi \in \Pi$, define  the augmented return
\[
R(\pi,\xi) = \sum_{t=1}^{T} \gamma^{t-1} \bar{r}_t(\xi).
\]
Assume that $\bar{r}_t(\xi) \in [0, \bar{r}_{\max}]$ for each $t$, so that
\[
R(\pi,\xi) \in [0, \bar{R}_{\max}].
\]

Define the expected return:
\[
J_{\text{test}}(\pi) = \mathbb{E}_{\xi \sim \mathcal{D}} \left[ R(\pi,\xi) \right],
\]
and the empirical return:
\[
J_{\text{train}}(\pi) = \frac{1}{n} \sum_{i=1}^{n} R(\pi,\xi_i).
\]

We aim to bound the expected generalization gap between the test return and empirical return for policies in class $\Pi$ via Rademacher complexity. Let the function class be defined as
\[
\mathcal{F} = \left\{ f_\pi(\xi) = R(\pi, \xi) \mid \pi \in \Pi \right\},
\]
where $R(\pi, \xi)$ is the total return over trajectory $\xi$ under policy $\pi$ using the augmented reward $\bar{r}_t$. Our goal is to bound:
\[
\sup_{\pi \in \Pi} \left| J_{\text{test}}(\pi) - J_{\text{train}}(\pi) \right| = \sup_{f \in \mathcal{F}} \left| \mathbb{E}[f(\xi)] - \frac{1}{n} \sum_{i=1}^{n} f(\xi_i) \right|.
\]

Let $\xi_1, \ldots, \xi_n$ be the training samples drawn i.i.d. from the environment distribution $\mathcal{D}$, and $\xi_1', \ldots, \xi_n'$ be another independent copy drawn from the same distribution. By using an independent ghost sample set and the triangle inequality, we have:
\begin{align*}
\mathbb{E}_{\{\xi_i\}} \left[ \sup_{f \in \mathcal{F}} \left( \mathbb{E}_{\xi \sim \mathcal{D}}[f(\xi)] - \frac{1}{n} \sum_{i=1}^{n} f(\xi_i) \right) \right]
&= \mathbb{E}_{\{\xi_i\}, \{\xi_i'\}} \left[ \sup_{f \in \mathcal{F}} \left( \frac{1}{n} \sum_{i=1}^{n} f(\xi_i') - f(\xi_i) \right) \right] \\
&\leq \mathbb{E}_{\{\xi_i\}, \{\xi_i'\}} \left[ \sup_{f \in \mathcal{F}} \frac{1}{n} \sum_{i=1}^{n} \left( f(\xi_i') - f(\xi_i) \right) \right].
\end{align*}

To simplify the expression, we now introduce independent Rademacher variables $\sigma_1, \ldots, \sigma_n \in \{-1, +1\}$, where each $\sigma_i$ takes value $+1$ or $-1$ with equal probability. Since $f(\xi_i') - f(\xi_i)$ is symmetric around zero due to $\xi_i \sim \xi_i'$, we can write:
\[
\mathbb{E}_{\{\xi_i\}, \{\xi_i'\}} \left[ \sup_{f \in \mathcal{F}} \frac{1}{n} \sum_{i=1}^{n} \left( f(\xi_i') - f(\xi_i) \right) \right]
= \mathbb{E}_{\{\xi_i\}, \{\xi_i'\}, \{\sigma_i\}} \left[ \sup_{f \in \mathcal{F}} \frac{1}{n} \sum_{i=1}^{n} \sigma_i \left( f(\xi_i') - f(\xi_i) \right) \right].
\]

We now apply the triangle inequality again:
\[
\sup_{f \in \mathcal{F}} \sum_{i=1}^{n} \sigma_i \left( f(\xi_i') - f(\xi_i) \right)
\leq \sup_{f \in \mathcal{F}} \sum_{i=1}^{n} \sigma_i f(\xi_i') + \sup_{f \in \mathcal{F}} \sum_{i=1}^{n} (-\sigma_i) f(\xi_i).
\]

Since $-\sigma_i$ is still a Rademacher variable and $\xi_i$ and $\xi_i'$ have the same distribution, the two expectations are equal. Thus, we obtain:
\begin{align*}
\mathbb{E}_{\{\xi_i\}, \{\xi_i'\}} \left[ \sup_{f \in \mathcal{F}} \frac{1}{n} \sum_{i=1}^{n} \left( f(\xi_i') - f(\xi_i) \right) \right]
&\leq 2 \mathbb{E}_{\{\xi_i\}, \{\sigma_i\}} \left[ \sup_{f \in \mathcal{F}} \frac{1}{n} \sum_{i=1}^{n} \sigma_i f(\xi_i) \right] \\
&= 2 \mathcal{R}_n(\mathcal{F}),
\end{align*}
where $\mathcal{R}_n(\mathcal{F})$ is the empirical Rademacher complexity of $\mathcal{F}$.


Assume every return is bounded,
$0 \le f_\pi(\xi) \le \bar{R}_{\max}$,
and that $f_\pi(\xi)$ is linear in the augmented per–step rewards
$ \bar{r}_t(\xi)$:
\[
f_\pi(\xi)
   \;=\;
\sum_{t=1}^{T} \gamma^{t-1} \bar{r}_t(\xi).
\]
Introduce the normalised return
$\tilde f_\pi(\xi) := f_\pi(\xi)\big/\bar{R}_{\max} \in [0,1]$
and let $\tilde{\mathcal{F}}
     := \{ \tilde f_\pi \mid \pi \in \Pi \}$.
Because Rademacher complexity is positively homogeneous in its
function class,
\[
\mathcal{R}_n(\mathcal{F})
   \;=\;
\mathcal{R}_n\!\bigl(\bar{R}_{\max}\,\tilde{\mathcal{F}}\bigr)
   \;=\;
\bar{R}_{\max}\,\mathcal{R}_n(\tilde{\mathcal{F}}).
\]
We measure the complexity of the policy class precisely through these
normalised returns and set
\[
\mathcal{R}_n(\Pi) := \mathcal{R}_n(\tilde{\mathcal{F}}).
\]

\textit{Justification.}
Even if the mapping $\pi \!\mapsto\! \tilde f_\pi$ is not injective,
Rademacher complexity is \textbf{monotone} with respect to set inclusion:
enlarging the function class can only increase
$\mathcal{R}_n$.
Hence analysing the (possibly larger) class
$\tilde{\mathcal{F}}$ yields a conservative upper bound
on the true policy complexity—exactly what we need for a valid
generalisation bound.

Combining the two displays yields
\[
\;
\mathcal{R}_n(\mathcal{F}) \;\le\;
\bar{R}_{\max}\,\mathcal{R}_n(\Pi)
\;
\]
(the identity can be written as “$\le$’’ because any alternative
normalisation would only shrink the right–hand side).

Substituting the above bound into the symmetrisation result, we obtain
\[
\mathbb{E}
\Bigl[
  \sup_{\pi \in \Pi}
  \bigl| J_{\mathrm{test}}(\pi) - J_{\mathrm{train}}(\pi) \bigr|
\Bigr]
\;\le\;
2\,\bar{R}_{\max}\,\mathcal{R}_n(\Pi),
\]

We now move from the expected generalization gap to a high-probability bound that holds uniformly over all policies $\pi \in \Pi$.

\vspace{0.5em}
Let $X_i = R(\pi, \xi_i) = \sum_{t=1}^{T} \gamma^{t-1} \bar{r}_t(\xi_i)$ be the augmented return of policy $\pi$ on the $i$-th training trajectory. Then $J_{\text{train}}(\pi) = \frac{1}{n} \sum_{i=1}^{n} X_i$ and $J_{\text{test}}(\pi) = \mathbb{E}_{\xi \sim \mathcal{D}}[X_i]$. By assumption, $X_i \in [0, \bar{R}_{\max}]$.

Applying Hoeffding's inequality for bounded i.i.d. variables, we have for any fixed $\pi \in \Pi$:
\[
\Pr\left( \left| J_{\text{test}}(\pi) - J_{\text{train}}(\pi) \right| \ge \varepsilon \right)
\le 2\exp\left( - \frac{2n\varepsilon^2}{(\bar{R}_{\max})^2} \right).
\]
Solving for $\varepsilon$ yields that with probability at least $1 - \delta$,
\begin{equation}
\left| J_{\text{test}}(\pi) - J_{\text{train}}(\pi) \right|
\le \bar{R}_{\max} \sqrt{\frac{\log(1/\delta)}{2n}}.
\tag{16}
\end{equation}

\vspace{0.5em}

Define the worst-case generalization gap over the policy class:
\[
\Delta(\mathcal{S}) := \sup_{\pi \in \Pi} \left( J_{\text{test}}(\pi) - J_{\text{train}}(\pi) \right),
\]
where $\mathcal{S} = \{\xi_1, \dots, \xi_n\}$ is the training set.

\textit{(i) Expected bound from above:} 
Using symmetrization and Rademacher complexity arguments, we already established:
\begin{equation}
\mathbb{E}_{\mathcal{S}}[\Delta(\mathcal{S})] \le 2 \bar{R}_{\max} \mathcal{R}_n(\Pi).
\label{eq:e_delta}
\end{equation}

\textit{(ii) High-probability deviation bound via McDiarmid’s inequality:}
Let us show that $\Delta(\mathcal{S})$ concentrates around its expectation. Consider replacing any single sample $\xi_i$ in $\mathcal{S}$ by an independent copy $\xi_i'$. Because each return $X_i = R(\pi, \xi_i)$ is bounded in $[0, \bar{R}_{\max}]$ and each contributes $\frac{1}{n}$ to the empirical mean, the influence of changing $\xi_i$ is bounded by:
\[
\left| \Delta(\mathcal{S}) - \Delta(\mathcal{S}^{(i)}) \right| \le \frac{\bar{R}_{\max}}{n}.
\]
Hence, $\Delta(\mathcal{S})$ is $\bar{R}_{\max}/n$-Lipschitz in each of its $n$ arguments.

Applying McDiarmid's inequality:
\[
\Pr\left( \Delta(\mathcal{S}) - \mathbb{E}[\Delta(\mathcal{S})] \ge \varepsilon \right)
\le \exp\left( - \frac{2\varepsilon^2}{\sum_{i=1}^n (\bar{R}_{\max}/n)^2} \right)
= \exp\left( - \frac{2n\varepsilon^2}{(\bar{R}_{\max})^2} \right).
\]
Solving for $\varepsilon$ again yields that with probability at least $1 - \delta$,
\begin{equation}
\Delta(\mathcal{S}) \le \mathbb{E}[\Delta(\mathcal{S})] + \bar{R}_{\max} \sqrt{\frac{\log(1/\delta)}{2n}}.
\label{eq:bound_max}
\end{equation}

\textit{(iii) Final generalization gap:} Combining Equation~\ref{eq:e_delta} and \ref{eq:bound_max}, with probability at least $1 - \delta$ over the random draw of the training set $\mathcal{S}$, we obtain:
\[
\sup_{\pi \in \Pi} \left[ J_{\text{test}}(\pi) - J_{\text{train}}(\pi) \right]
\le 2 \bar{R}_{\max} \mathcal{R}_n(\Pi) + \bar{R}_{\max} \sqrt{\frac{\log(1/\delta)}{2n}}.
\]
Equivalently, for all $\pi \in \Pi$,
\begin{equation}
 J_{\text{test}}(\pi) - J_{\text{train}}(\pi)  \leq 2 \bar{R}_{\max}\, \mathcal{R}_n(\Pi) + \bar{R}_{\max} \sqrt{\frac{\log(1/\delta)}{2n}}
\label{eq:gen_final}
\end{equation}
\end{proof}

\noindent\textbf{Conclusion.} Equation~\ref{eq:gen_final} provides a uniform generalization gap for any policy $\pi \in \Pi$, showing that the expected test-time performance is lower bounded by the training performance minus a complexity-dependent regularization term. According to this theorem, as the augmented reward $R_{\rm score}(s_t)$ is well-aligned with genuine logical reasoning, it acts as a regularizer that effectively reduces the Rademacher complexity $\mathcal{R}_n(\Pi)$, thereby tightening the bound. This theoretical result highlights that our proposed process supervision framework not only improves credit assignment during training but also enhances generalization to unseen reasoning tasks.

\section{Detailed Implementation of GRPO-R}
\label{app:grpo_r_details}
Our proposed process-level reasoning score supervision is compatible with any token-level RL algorithm. In this work, we instantiate it within Group Relative Policy Optimization (GRPO), yielding \textbf{GRPO-R}. GRPO is a scalable and widely used RL framework for reasoning model training, which promotes the generation of high-quality reasoning trajectories by ranking $G$ candidate outputs based on their relative returns, without relying on explicit value estimation~\cite{DeepSeekAI2025,shao2024deepseekmath}. 

Given a prompt $q$ and $G$ outputs $\{o_i\}_{i=1}^G$, each output $o_i$ corresponds to a sequence of reasoning states $\{s_{i,1}, \dots, s_{i,K}\}$ produced over $K$ reasoning steps. In the original GRPO setup, only the final step receives a nonzero reward:
\[
r_i^{\mathrm{step}}(j)
=
\begin{cases}
r_i^{\mathrm{final}}, & j = K,\\
0, & j < K,
\end{cases}
\]
where \( r_i^{\mathrm{final}} \) denotes the scalar reward assigned to the final outcome.

We replace this sparse signal with our shaped step-level reward using potential-based reward shaping:
\[
\bar{r}_i^{\mathrm{step}}(j)
=\;\tilde r_i^{\mathrm{step}}(j)
\;-\;\gamma\,\tilde R_{\mathrm{score}}(s_{i,j+1})
\;+\;\tilde R_{\mathrm{score}}(s_{i,j}),
\]
where \( \tilde R_{\mathrm{score}}(s) = \min\bigl(R_{\mathrm{score}}(s), \tau\bigr) \) and we set \( \gamma = 1 \). These shaped rewards are collected into the set \( \mathbf{R}' \), standardized as:
\[
\hat r_i^{\mathrm{step}}(j)
=\;\frac{\bar{r}_i^{\mathrm{step}}(j)-\text{mean}(\mathbf{R}')}
{\text{std}(\mathbf{R}')},
\]
and used to compute token-level advantages:
\[
\hat A_{i,t}=\sum_{j:\,\mathrm{step}(j)\ge t}\hat r_i^{\mathrm{step}}(j).
\]

Finally, we optimize the policy using the enhanced GRPO objective, termed GRPO-R:
\begin{align}
\mathcal{J}_{\text{GRPO-R}}(\theta) = 
\mathbb{E}_{q \sim P(Q),\, \{o_i\} \sim \pi_{\theta_{\text{old}}}(O \mid q)} 
\Bigg[
\sum_{i=1}^{G} \sum_{t=1}^{|o_i|} 
&\min \Big(  \frac{ \pi_\theta(o_{i,t} \mid q, o_{i,<t}) }{ \pi_{\theta_{\text{old}}}(o_{i,t} \mid q, o_{i,<t}) } \, \hat A_{i,t}, \nonumber \\
&\quad \text{clip}\big(  \frac{ \pi_\theta(o_{i,t} \mid q, o_{i,<t}) }{ \pi_{\theta_{\text{old}}}(o_{i,t} \mid q, o_{i,<t}) }, 1 - \epsilon, 1 + \epsilon \big) \, \hat A_{i,t} \Big) \nonumber \\
&\quad - \beta \cdot \mathrm{D}_{\mathrm{KL}} \left[ \pi_\theta \Vert \pi_{\text{ref}} \right]
\Bigg].
\end{align}

\begin{table}
\centering
\caption{Statistics of ReTruthQA dataset across domains.}
\label{tab:reasoning_statistics}
\begin{tabular}{l|cccc}
\toprule
   Dataset &  \#Samples &  \#Traces &  Avg Truthful Traces &  Avg Hallucination Traces \\
\midrule
      MATH &       57 &     417 &     3.35 &     3.96 \\
   Science &       88 &     541 &     3.05 &     3.10 \\
MultiHopQA &      184 &    1186 &     2.74 &     3.70 \\
\bottomrule
\end{tabular}
\end{table}

\section{ReTruthQA Construction}
\label{app:retruthqa}
\subsection{Data Sources and Models} 
Due to the lack of dedicated datasets for evaluating reasoning hallucination detection methods, especially for strong open-source models such as \texttt{DeepSeek-R1-7B} and \texttt{R1-14B}, we construct a new benchmark tailored for hallucination detection in multi-step reasoning tasks.
We select three major categories of reasoning tasks: \texttt{Math}, \texttt{Science}, and \texttt{MultiHopQA}. 

For \texttt{Math}, we construct the dataset using benchmark datasets commonly used for evaluating mathematical reasoning capabilities, including \texttt{MATH500}~\citep{lightman2023let}, \texttt{AMC 2023}~\citep{aimo_validation_amc}, and \texttt{AIME 2024}~\citep{aimo_validation_aime}. 

For \texttt{Science}, we adopt \texttt{GPQA}~\citep{rein2024gpqa}, a PhD-level science multiple-choice QA dataset with questions authored by domain experts in physics, chemistry, and biology.

For \texttt{MultiHopQA}, we randomly sample 1000 questions from four multi-hop QA datasets: \texttt{HotpotQA}~\citep{yang2018hotpotqa}, \texttt{2WikiMultihopQA}~\citep{ho2020constructing}, \texttt{MuSiQue}~\citep{trivedi2022musique}, and \texttt{Bamboogle}~\citep{press2022measuring}.

For each question, we generate 20 responses using \texttt{DeepSeek-R1-Distill-Qwen-7B} and \texttt{DeepSeek-R1-Distill-Qwen-14B} via random sampling. The prompting format is as follows:

\textbf{Math:}
\begin{tcolorbox}[colback=gray!20, colframe=gray, coltitle=white, fonttitle=\bfseries]
Please answer the following math question.  \\
You should provide your final answer in the format \texttt{\textbackslash boxed\{YOUR\_ANSWER\}}. \\
Separate your following steps using \texttt{\textbackslash n\textbackslash n}. \\
\texttt{Question:\textbackslash n\textbackslash n}
\end{tcolorbox}

\textbf{Science:}
\begin{tcolorbox}[colback=gray!20, colframe=gray, coltitle=white, fonttitle=\bfseries]
Please answer the following multiple-choice question.  \\
You should provide your final choice in the format \texttt{\textbackslash boxed\{YOUR\_CHOICE\}}. \\
Separate your following steps using \texttt{\textbackslash n\textbackslash n}. \\
\texttt{Question:\textbackslash n\textbackslash n}
\end{tcolorbox}

\textbf{MultiHopQA:}
\begin{tcolorbox}[colback=gray!20, colframe=gray, coltitle=white, fonttitle=\bfseries]
Please answer the following question.  \\
You should provide your final answer in the format \texttt{\textbackslash boxed\{YOUR\_ANSWER\}}. \\
Separate your following steps using \texttt{\textbackslash n\textbackslash n}. \\
\texttt{Question:\textbackslash n\textbackslash n}
\end{tcolorbox}

\subsection{Reasoning Step Segmentation Strategy}
\label{app:step_seg}
We adopt a two-stage segmentation procedure. First, we split the reasoning trace based on cognitive behavior tokens such as \texttt{</think>}, \texttt{Wait}, \texttt{But}, \texttt{However}, \texttt{Hmm}, \texttt{Alternatively}, which typically mark transitions in reasoning patterns. Then, we apply a finer-grained split based on formatting: as specified in the prompt, the LRM is instructed to separate reasoning steps using \texttt{\textbackslash n\textbackslash n}, which we use as a delimiter. This hybrid approach ensures both rule-based and model-aligned step boundaries.

\subsection{Annotation Process}

\textbf{1. Automatic hallucination trace identification.}  
To ensure precision and avoid noise caused by random model errors, a reasoning trace is labeled as hallucinated only if its rollout becomes incorrect with a failure rate exceeding 90\% from a specific reasoning step onward, measured over 16 rollouts.  We adopt a binary search–style trace slicing procedure inspired by OmegaProcess~\citep{luo2024improve} to efficiently identify hallucination points. This strategy ensures stability and causality in hallucination step detection, avoiding incidental errors due to sampling randomness. For the Science domain, which mainly consists of multiple-choice questions and may contain correct guesses, we additionally perform multiple random rollouts for traces with correct answers to ensure a success rate above 90\% before labeling them as truthful.

\textbf{2. Filtering non-hallucination failures.}  
We use GPT-4o-Mini to exclude samples where the incorrect final answer is due to clearly flawed or illogical reasoning, which does not satisfy our definition of hallucination (i.e., coherent and persuasive chains with underlying logical or factual errors). The filtering prompt is:

\begin{tcolorbox}[colback=gray!20, colframe=gray, coltitle=white, fonttitle=\bfseries]
Please evaluate if the following reasoning for the given question is logically sound and leads to a correct solution. \\
Only respond with a score between 0 and 1, where: \\
0: completely incorrect or illogical reasoning \\
1: perfectly sound and correct reasoning \\

\texttt{Question: \{question\}} \\
\texttt{Reasoning: \{reasoning\}} \\
Score (0--1):
\end{tcolorbox}

\begin{figure}[t]
    \centering
    \includegraphics[width=0.5\textwidth]{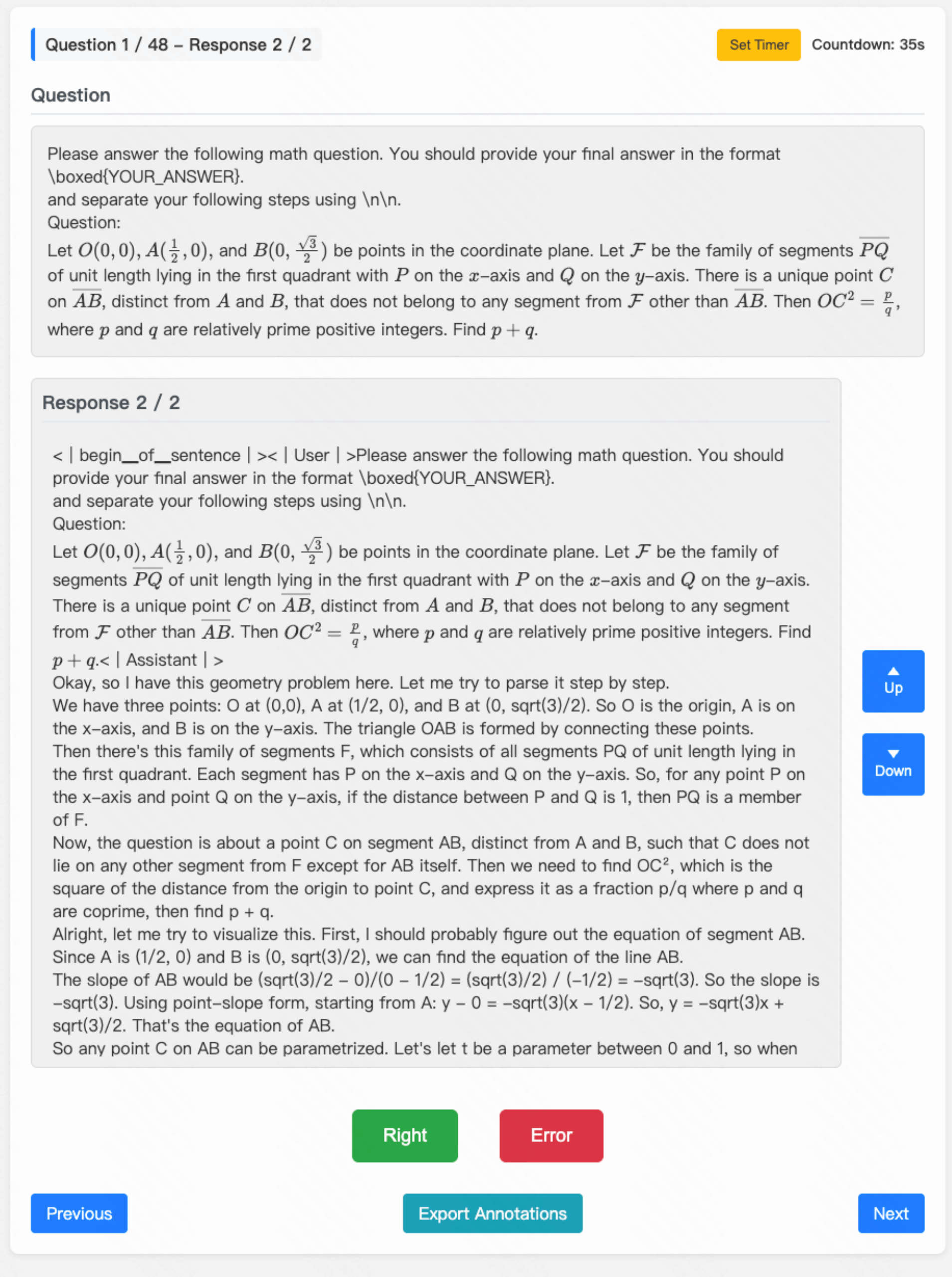}
    \caption{Interface Display of the Data Annotation Platform.}
    \label{fig:data_annotation}
\end{figure}

\textbf{3. Human validation.}  
We further perform human annotation to verify borderline cases. Two annotators with at least undergraduate-level backgrounds in computer science independently assess whether the reasoning trace is valid. We developed a web-based annotation platform with a timer (\autoref{fig:data_annotation}) to standardize reading time. Based on average reading speeds (200–300 wpm for academic text), and trace lengths (typically 2000–3000 words), we set the following maximum judgment times: (1) MultiHopQA: 3 minutes
  (2) Math: 5 minutes
  (3) Science: 8 minutes

Annotators must determine within the allotted time whether a reasoning trace contains hallucinations. If they fail to identify an error in time, the trace is labeled as correct. Cases judged correct by humans but verified to be incorrect are labeled as hallucinations, ensuring that the resulting dataset captures only traces that genuinely mislead users, which is aligned with the definition of reasoning hallucination.

Final dataset statistics are shown in \autoref{tab:reasoning_statistics}. For the \textbf{Multi-Trace Ranking Setting}, we directly use the collected hallucinated and truthful responses. For the \textbf{Binary Detection Setting}, which focuses on single-response accuracy, we retain one hallucinated and one truthful response per question to reflect more realistic ad-hoc usage scenarios.

\section{GSM-NoOp Construction Process}
\label{app:GSM-NoOp}
Following the construction procedure proposed in~\cite{mirzadeh2024gsm}, we randomly sample 300 examples from the GSM8K dataset. For each question, we use \texttt{GPT-4o} to generate a No-Op phrase using the following prompt:

\begin{tcolorbox}[colback=gray!20, colframe=gray, coltitle=white, fonttitle=\bfseries]
Given the following math question, generate a seemingly relevant but ultimately 
inconsequential statement (No-Op) that can be added to the question without affecting its solution.

Question: \{Question\}

Generate a No-Op statement that: \\
1. Is short and concise \\
2. Seems relevant to the context \\
3. Does not affect the mathematical reasoning \\
4. Is natural and fits grammatically

No-Op statement:
\end{tcolorbox}

We then use \texttt{GPT-4o} to combine the generated No-Op phrase with the original question using the following prompt:

\begin{tcolorbox}[colback=gray!20, colframe=gray, coltitle=white, fonttitle=\bfseries]
Please combine the following math question and No-Op phrase into a single, 
natural-sounding question. The No-Op phrase should be integrated smoothly 
without changing the mathematical meaning.

Math Question: \{Question\}
No-Op Phrase: \{NoOp Phrase\}

Combined Question:
\end{tcolorbox}

The merged questions form our constructed \textbf{GSM-NoOp} dataset.

To evaluate whether the generated reasoning steps are misled by the inserted No-Op phrase, we prompt \texttt{GPT-4o} with the following instruction:

\begin{tcolorbox}[colback=gray!20, colframe=gray, coltitle=white, fonttitle=\bfseries]
Please evaluate if the following reasoning step is being misled by the given No-Op phrase. 
Provide a score between 0 and 1, where:\\
a. 0 means the step is not misled by the No-Op phrase at all\\
b. 1 means the step is completely misled by the No-Op phrase\\
c. Values in between indicate partial misleading\\

Note: Simply mentioning the No-Op phrase does not count as being misled. 
If the step mentions the No-Op phrase but explicitly rejects or explains why it is irrelevant to solving the problem, this should be scored as 0.

Reasoning step: \{Reasoning Step\}
No-Op phrase: \{NoOp Phrase\}

Please provide only a number between 0 and 1, with up to 2 decimal places, 
wrapped in \textbackslash boxed\{\}. For example: \textbackslash boxed\{0.85\}
\end{tcolorbox}

\section{Details of Understanding the Mechanisms Behind Reasoning Hallucination Patterns}
\label{app:detail_pattern_analysis}
In this section, we focus on analyzing the underlying cause of \textbf{Pattern \#1}, as \textbf{Pattern \#2} has already been explained through the attention behavior of LRMs in the previous section. Pattern \#1 highlights that hallucinated reasoning traces tend to exhibit larger fluctuations in reasoning depth, particularly in the early steps. Inspired by our preliminary analysis in \S~\ref{sec:case_study}, we hypothesize that this may stem from the model’s built-in verification capability. However, several key questions remain: \textbf{Q1}: What triggers verification behavior in LRMs? \textbf{Q2}: Do excessively high reasoning scores genuinely indicate overthinking? \textbf{Q3}: If Q2 holds, what factors lead to the emergence of such overthinking steps?

To answer these questions, we construct reasoning step triples \((c_1, c_2, c_3)\) with different properties drawn from reasoning traces:
\textbf{Stable}: The first type consists of triples from truthful traces where adjacent steps differ in \( R_{\text{score}} \) by less than 0.1, representing stable reasoning.
\textbf{Rising-1}: The second type contains hallucinated triples where \( R_{\text{score}}(c_3) - R_{\text{score}}(c_2) > 1 \) and \( R_{\text{score}}(c_3) < 4 \), used to analyze verification triggered by shallow pattern-matching.
\textbf{Rising-2}: The third type is similar to Rising-1 but with \( R_{\text{score}}(c_3) > 4 \), aimed at understanding overthinking induced by verification.

\paragraph{Analysis.}
To investigate \textbf{Q1}, we analyze whether reasoning steps \( c_1 \) and \( c_2 \) in the stable and rising (Rising-1 + Rising-2) triples are logically consistent, using GPT-4o as the judge (prompt details in Appendix~\ref{app:prompt_hall_pattern}). As shown in \autoref{fig:emp_2}(a), the stable triples exhibit significantly higher consistency between \( c_1 \) and \( c_2 \) than rising triples, indicating that LRMs are more likely to trigger verification when early steps are internally inconsistent.

To examine \textbf{Q2}, we evaluate the correctness of \( c_2 \) and \( c_3 \) in Rising-2 triples. Using ground-truth answers and GPT-4o-based annotation (prompt details in Appendix~\ref{app:prompt_hall_pattern}), we assess whether these steps are logically aligned with the ground-truth answers. As shown in \autoref{fig:emp_2}(b), \( c_2 \) in Rising-2 triples is substantially more accurate than \( c_3 \), confirming that verification in this case often modifies correct reasoning into incorrect steps. These findings support the hypothesis that excessively high \( R_{\text{score}} \) values in hallucinated reasoning traces are symptomatic of overthinking—steps that exhibit apparent reasoning depth but in fact reflect spurious or detrimental reasoning.

To address \textbf{Q3}, we analyze the relationship between perplexity and \( R_{\text{score}} \). Specifically, we randomly sample 200 reasoning steps from ReTruthQA and compute their perplexities as follows:
\begin{align}
\operatorname{PPL}(c_k) &= \exp\left( - \frac{1}{|c_k|} \sum_{t_{m+1}^{k} \in c_k} \log p\left(t_{m+1}^{k} \mid t_{\leq m}^{k}\right) \right), \label{eq:ppl} \\
\operatorname{PPL}(\mathcal{C}) &= \left< \operatorname{PPL}(c_1), \operatorname{PPL}(c_2), \dots, \operatorname{PPL}(c_K) \right>. \label{eq:ppl_trace}
\end{align}
where \( p(t_{m+1}^{k} \mid t_{\leq m}^{k}) \) denotes the model's predicted probability for token \( t_{m+1}^{k} \) given the prefix \( t_{\leq m}^{k} \) within the reasoning trace.

As shown in \autoref{fig:emp_2}(c), perplexity and \( R_{\text{score}} \) are strongly negatively correlated—steps with higher reasoning depth tend to have lower perplexity, which is intuitive since deep reasoning often yields more predictable outputs. However, when comparing the final step \( c_3 \) across stable and Rising-2 triples, we find an interesting phenomenon in \autoref{fig:emp_2}(d): despite having higher \( R_{\text{score}} \), \( c_3 \) in Rising-2 triples has higher perplexity than in stable triples. This suggests that overthinking steps induced by an incorrect verification result in an uncertain or internally unstable generation.

We hypothesize that such overthinking may reflect \textit{spurious verification}—a behavior where the model performs superficial or misguided validation in pursuit of higher reward during RL fine-tuning. This behavior can persist through distillation into smaller models, propagating reasoning hallucinations. Based on this analysis, we identify a third hallucination pattern:
\textbf{Pattern \#3:} Overthinking reasoning steps exhibit a positive correlation between \( R_{\text{score}} \) and perplexity (PPL).

\section{Prompt for Hallucination Patterns Analysis}
\label{app:prompt_hall_pattern}

Prompt for step consistency analysis of \textbf{Q1}: 
\begin{tcolorbox}[colback=gray!20, colframe=gray, coltitle=white, fonttitle=\bfseries]
Please evaluate whether the following reasoning step introduces a new solution approach compared to the preceding steps.  
Respond with a score of 0 or 1, where: \\
0: The step follows the same solution approach as the previous steps. \\
1: The step explores a new solution approach or direction.

Reasoning step:  
\{step content\}

Previous steps:  
\{step content\}

Score (0/1):  
\end{tcolorbox}

Prompt for step correctness analysis of \textbf{Q2}: 
\begin{tcolorbox}[colback=gray!20, colframe=gray, coltitle=white, fonttitle=\bfseries]
Please evaluate whether the following reasoning step aligns with the final answer.  
Respond with a score of 0 or 1, where: \\
0: The step is inconsistent with the final answer. \\
1: The step is consistent with the final answer.

Reasoning step:  
\{step\}

Final answer:  
\{answer\}

Score (0/1):  
\end{tcolorbox}

\section{More Results of Reasoning Hallucination Pattern Analysis}
\label{app:more_result_reason_hall}
\begin{figure}[t]
    \centering
    \begin{subfigure}{0.32\textwidth}
        \includegraphics[width=\linewidth]{graphs/cv_scatter_math.pdf}
    \end{subfigure}
    \hfill
    \begin{subfigure}{0.32\textwidth}
        \includegraphics[width=\linewidth]{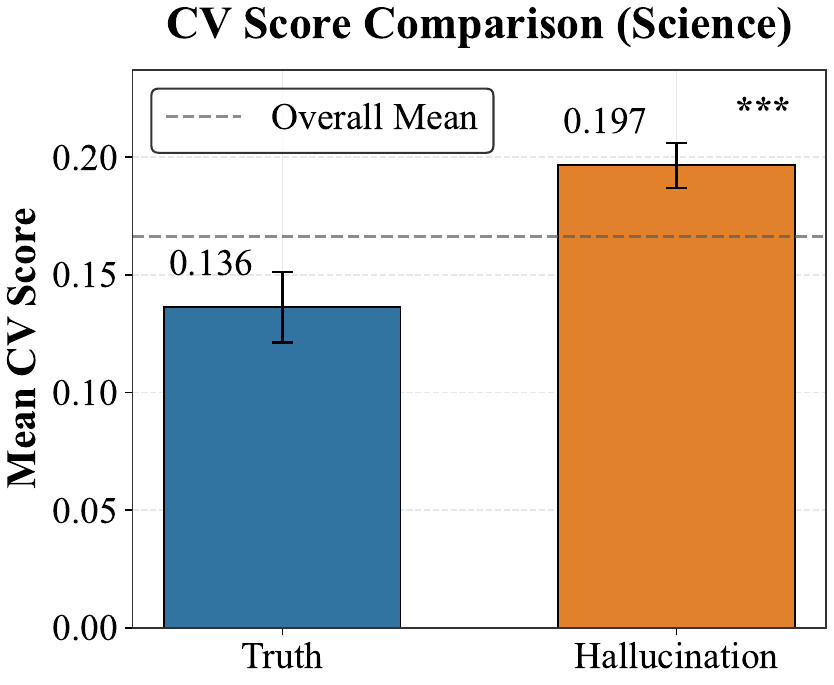}
    \end{subfigure}
    \hfill
    \begin{subfigure}{0.32\textwidth}
        \includegraphics[width=\linewidth]{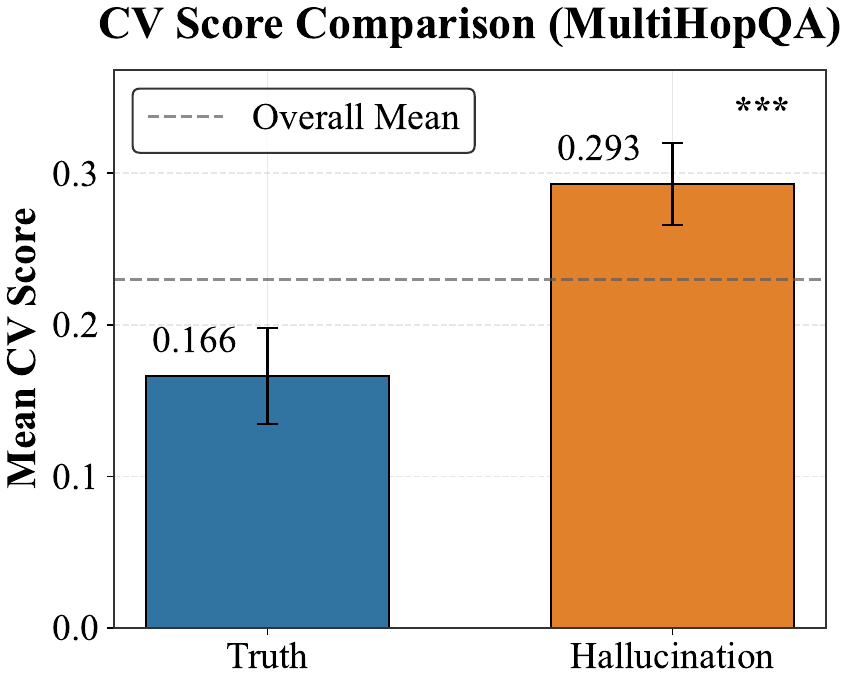}
    \end{subfigure}
    
    \begin{subfigure}{0.31\textwidth}
        \includegraphics[width=\linewidth]{graphs/attention_scores_math.pdf}
    \end{subfigure}
    \hfill
    \begin{subfigure}{0.31\textwidth}
        \includegraphics[width=\linewidth]{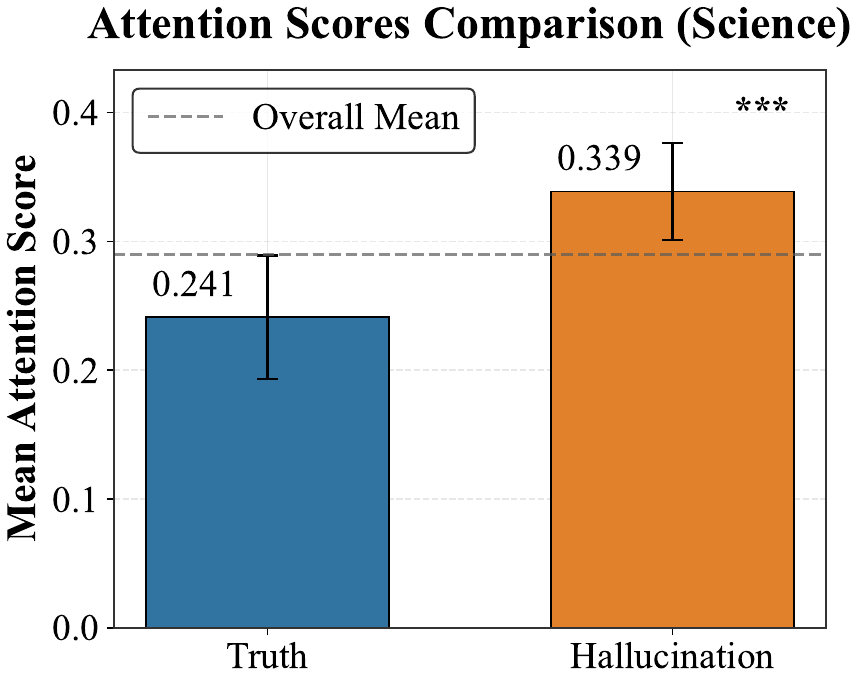}
    \end{subfigure}
    \hfill
    \begin{subfigure}{0.33\textwidth}
        \includegraphics[width=\linewidth]{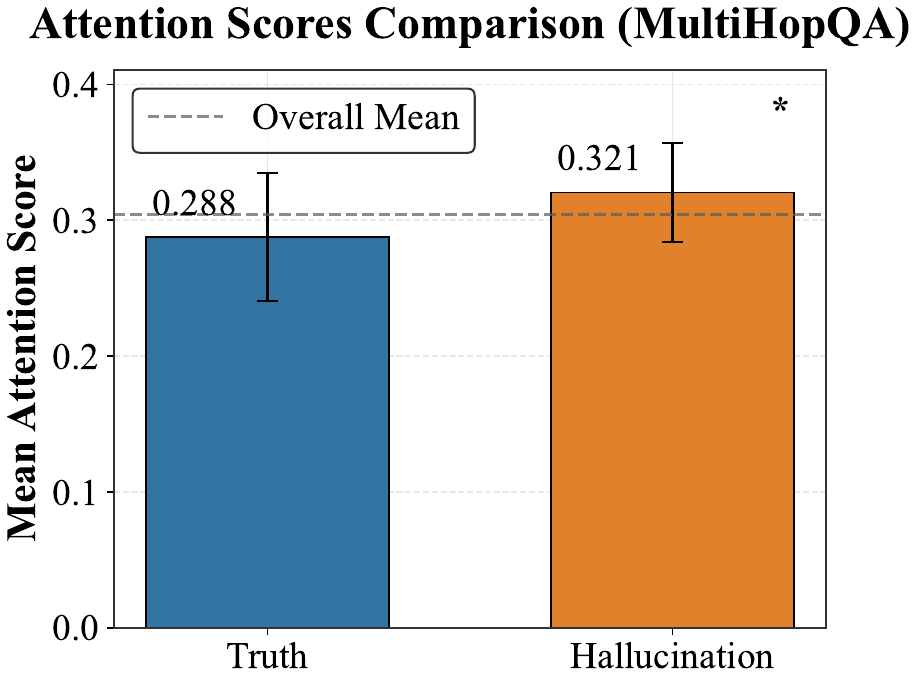}
    \end{subfigure}
    
    \caption{Evaluation of Pattern~\#1 and Pattern~\#2 on ReTruthQA. Asterisks indicate statistical significance based on a t-test: * for $p$-value < 0.05, and *** for $p$-value < 0.001.}
    \label{fig:emp_1_result_more}
\end{figure}
The hyperparameter settings involved in Section~\ref{sec:reason_analysis} are as follows. The constant \( r \), which controls the size of the early-step window, is empirically set to \( r = 2 \). The constant \( \eta \), which defines the portion of late reasoning steps, is set to \( \eta = 0.75 \). The constant \( K \), used in computing attention to earlier steps, is set to \( K = 5 \). The threshold \( \tau \) for identifying potentially overthinking steps is set to \( \tau = 4 \). These hyperparameters are derived from case analysis and are applied consistently throughout the subsequent reasoning hallucination detection and mitigation experiments.

The validity of Pattern \#1 and Pattern \#2 is verified across all domains of ReTruthQA, with experimental results shown in Figure~\ref{fig:emp_1_result_more}, where across all three domains, hallucinated reasoning traces consistently exhibit significantly higher CV scores and Attention scores than truthful traces.

\section{Evaluation and Baseline Details of Reasoning Hallucination Detection}
\label{app:retruthqa_eval}
Based on ReTruthQA, we design two evaluation settings for RHD model: \textbf{(1) Binary Detection Setting:} This setting assesses the model’s ability to detect hallucinations in individual question-reasoning pairs $(Q, C)$, measuring detection performance using the Area Under the ROC Curve (\textit{\textbf{AUC}}) and Pearson Correlation Coefficient (\textit{\textbf{PCC}});
\textbf{(2) Multi-Trace Ranking Setting:} This setting evaluates the model’s ability to identify the truthful answer among multiple reasoning traces for the same question $(Q, \{C_1, C_2, \ldots, C_3\})$. We follow the evaluation setup of TruthfulQA-MC~\citep{lin2021truthfulqa}, and report the following metrics: \textit{\textbf{MC1}}: The percentage of instances where the hallucination score of the most hallucinated reasoning trace exceeds that of all truthful traces;
    \textit{\textbf{MC2}}: The normalized total hallucination score assigned to the hallucinated reasoning traces;
    \textit{\textbf{MC3}}: The percentage of hallucinated reasoning traces that receive a higher hallucination score than all truthful traces.
These metrics collectively measure the ranking quality of hallucination detection in multi-sample generation settings.

For baselines, we consider the following categories:
\textbf{(1) Ensemble-based self-evaluation methods}, where hallucination scores are obtained through repeated generation, self-verification, or peer voting among LLMs. This category includes ChainPoll~\citep{friel2023chainpoll}, LMvLM~\citep{cohen2023lm}, and SelfCheckGPT~\citep{manakul2023selfcheckgpt}.
\textbf{(2) Uncertainty-based methods}, which estimate hallucination likelihood based on model uncertainty, including P(True)~\citep{kadavath2022language}, LN-Entropy~\citep{ren2022out}, and Perplexity (PPL)~\citep{malinin2020uncertainty}.
\textbf{(3) Self-awareness-based methods}, which rely on internal model representations to detect hallucinations, such as UQAC~\citep{li2025languagemodeluncertaintyquantification} and EigenScore~\citep{cheninside}.
\textbf{(4) LLM-as-Critic models}, including \texttt{GPT-4o}~\citep{achiam2023gpt} and \texttt{Qwen2.5-32B}~\citep{yang2024qwen2}, which act as external evaluators of reasoning traces.
\textbf{(5) Process reward models}, such as \texttt{Qwen2.5-Math-7B-PRM800K}~\citep{processbench} and \texttt{Qwen2.5-Math-PRM-7B}~\citep{prmlessons}, trained with step-level supervision for reasoning evaluation.
\textbf{(6) Length-based scoring}, motivated by recent findings that longer reasoning traces are more prone to hallucinations~\citep{zeng2025revisiting}, we include \texttt{Length-Score}, which directly uses the length of the reasoning trace as its hallucination score.

\section{Implementation Details for Reasoning Hallucination Detection}
\label{app:rhd_exp}
We conduct all experiments on machines equipped with NVIDIA A6000 GPUs and 52-core Intel(R) Xeon(R) Gold 6230R CPUs running at 2.10GHz. We utilize the Huggingface \texttt{Transformers} and \texttt{TRL} libraries to implement and run our experiments. During response generation, we use random sampling with a temperature of 0.7 and a maximum decoding length of 15{,}000 tokens for Math tasks and 10{,}000 tokens for all other tasks. For Reasoning Hallucination Detection (RHD), we perform two-fold validation to select optimal hyperparameters, while baselines are tuned within the ranges specified in their original works. To ensure stability, all randomized experiments are repeated three times and the average results are reported.

We perform a grid search over the interval $[0, 1]$ with a step size of 0.1 to determine the best combination of reasoning score weights $\alpha_1$, $\alpha_2$, $\alpha_3$, and $\alpha_4$ using two-fold validation to select the hyperparameters. For R1-7B, the best weights in the \texttt{Math} domain are $\alpha_1 = 0$, $\alpha_2 = 0.4$, $\alpha_3 = 0$, and $\alpha_4 = 0.3$ for the Multi-Trace Ranking setting, and $\alpha_1 = 0$, $\alpha_2 = 0.9$, $\alpha_3 = 0.8$, and $\alpha_4 = 0.4$ for the Binary Detection setting. In the \texttt{Science} domain, the best weights are $\alpha_1 = 0.1$, $\alpha_2 = 1.0$, $\alpha_3 = 0$, and $\alpha_4 = 0$ for Multi-Trace Ranking, and $\alpha_1 = 0$, $\alpha_2 = 0.7$, $\alpha_3 = 0.2$, and $\alpha_4 = 0$ for Binary Detection. In the \texttt{MultiHopQA} domain, the best weights are $\alpha_1 = 0.4$, $\alpha_2 = 0.1$, $\alpha_3 = 0.6$, and $\alpha_4 = 0.4$ for Multi-Trace Ranking, and $\alpha_1 = 0$, $\alpha_2 = 0$, $\alpha_3 = 0.3$, and $\alpha_4 = 0$ for Binary Detection.

For R1-14B, the best weights in the \texttt{Math} domain are $\alpha_1 = 0.3$, $\alpha_2 = 0.7$, $\alpha_3 = 0.1$, and $\alpha_4 = 0.1$ for Multi-Trace Ranking, and $\alpha_1 = 0$, $\alpha_2 = 0.3$, $\alpha_3 = 1.0$, and $\alpha_4 = 0.2$ for Binary Detection. In the \texttt{Science} domain, we obtain $\alpha_1 = 0$, $\alpha_2 = 0.5$, $\alpha_3 = 0.5$, and $\alpha_4 = 0.1$ for Multi-Trace Ranking, and $\alpha_1 = 0$, $\alpha_2 = 0$, $\alpha_3 = 0.8$, and $\alpha_4 = 0.1$ for Binary Detection. In the \texttt{MultiHopQA} domain, the optimal weights are $\alpha_1 = 0.7$, $\alpha_2 = 0.9$, $\alpha_3 = 0.1$, and $\alpha_4 = 0.0$ for Multi-Trace Ranking, and $\alpha_1 = 1.0$, $\alpha_2 = 0$, $\alpha_3 = 0.1$, and $\alpha_4 = 0.1$ for Binary Detection.

Candidate reasoning score layers $\mathcal{J}$ are selected from $\{14, 16, 18, 20, 22, 24, 26\}$ for R1-7B and from $\{32, 36, 40, 42, 44, 46\}$ for R1-14B, while attention score layers $\mathcal{L}$ are fixed across models as $\{1, 3, 5, 7, 9, 11, 13\}$. The models used in our experiments, \texttt{DeepSeek-R1-Distill-Qwen-7B} and \texttt{DeepSeek-R1-Distill-Qwen-14B}, are publicly available at \url{https://huggingface.co/deepseek-ai/DeepSeek-R1-Distill-Qwen-7B} and \url{https://huggingface.co/deepseek-ai/DeepSeek-R1-Distill-Qwen-14B}, respectively.

\begin{table}[t]
\centering
\resizebox{0.99\linewidth}{!}{
\renewcommand\arraystretch{1.15}
\setlength{\tabcolsep}{1.5mm}
\begin{tabular}{l|l|ccc|ccc|ccc}
\toprule
\multirow{2}{*}{\textbf{Model}} & \multirow{2}{*}{\textbf{Variant}} & \multicolumn{3}{c|}{\textbf{MATH}} & \multicolumn{3}{c|}{\textbf{Science}} & \multicolumn{3}{c}{\textbf{MultiHopQA}} \\
\cmidrule(lr){3-5} \cmidrule(lr){6-8} \cmidrule(lr){9-11}
~ & ~ & \textbf{MC1} & \textbf{MC2} & \textbf{MC3} & \textbf{MC1} & \textbf{MC2} & \textbf{MC3} & \textbf{MC1} & \textbf{MC2} & \textbf{MC3} \\
\midrule
\multirow{5}{*}{R1-7B} 
& RHD & 0.6591 & 0.4765 & 0.5699 & 0.6207 & 0.5448 & 0.6009 & 0.7660 & 0.6255 & 0.7103 \\
& RHD (w/o $\operatorname{Avg}(\mathcal{R}_{\text{score}})$) & 0.6591 & 0.4765 & 0.5699 & 0.6128 & 0.5307 & 0.5934 & 0.7383 & 0.6032 & 0.7082 \\
& RHD (w/o CV Score) & 0.6364 & 0.4663 & 0.5330 & 0.4483 & 0.3862 & 0.4977 & 0.7447 & 0.6043 & 0.6996 \\
& RHD (w/o Attention Score) & 0.6591 & 0.4765 & 0.5699 & 0.6207 & 0.5448 & 0.6009 & 0.6383 & 0.5372 & 0.6123 \\
& RHD (w/o PCC Score) & 0.5909 & 0.3830 & 0.5210 & 0.6207 & 0.5448 & 0.6009 & 0.6809 & 0.5553 & 0.6323 \\
\midrule
\multirow{5}{*}{R1-14B} 
& RHD & 0.3692 & 0.3005 & 0.4644 & 0.6667 & 0.4714 & 0.5671 & 0.5785 & 0.4421 & 0.5154 \\
& RHD (w/o $\operatorname{Avg}(\mathcal{R}_{\text{score}})$) & 0.3538 & 0.2867 & 0.4847 & 0.7241 & 0.4609 & 0.5531 & 0.5589 & 0.4284 & 0.5290 \\
& RHD (w/o CV Score) & 0.3692 & 0.2882 & 0.4725 & 0.6470 & 0.4484 & 0.5332 & 0.5455 & 0.4273 & 0.5403 \\
& RHD (w/o Attention Score) & 0.3231 & 0.2692 & 0.4503 & 0.6724 & 0.4511 & 0.5190 & 0.5702 & 0.4322 & 0.5180 \\
& RHD (w/o PCC Score) & 0.3692 & 0.2882 & 0.4725 & 0.6724 & 0.4601 & 0.5683 & 0.5785 & 0.4421 & 0.5154 \\
\bottomrule
\end{tabular}
}
\caption{Ablation study of the RHD model on three different domains of ReTruthQA. Each row removes one component of the hallucination score.}
\label{tab:rhd-ablation}
\end{table}

\section{Ablation Study of RHD}
\label{app:ablation_rhd}
In this section, we analyze the contribution of each module within the RHD model to reasoning hallucination detection. 
As shown in Table~\ref{tab:rhd-ablation}, removing any single component leads to a significant performance drop on most datasets in the Reasoning Hallucination Detection task. This validates the effectiveness of adopting a multivariate regression formulation, where all components jointly serve as covariates. Although some coefficients may appear less influential in certain domains, they demonstrate notable impact in others. This observation suggests that different domains exhibit distinct hallucination pattern preferences, further supporting the validity of the empirically discovered patterns, which can be effectively leveraged for reasoning hallucination detection.

\begin{table}[t]
\centering
\resizebox{0.85\linewidth}{!}{
\renewcommand\arraystretch{1.15}
\setlength{\tabcolsep}{1.5mm}
\begin{tabular}{l|ccc|ccc|ccc}
\toprule
\multirow{2}{*}{\textbf{Layers}} & \multicolumn{3}{c|}{\textbf{Math}} & \multicolumn{3}{c|}{\textbf{Science}} & \multicolumn{3}{c}{\textbf{MultiHopQA}} \\
\cmidrule(lr){2-4} \cmidrule(lr){5-7} \cmidrule(lr){8-10}
~ & \textbf{MC1} & \textbf{MC2} & \textbf{MC3} & \textbf{MC1} & \textbf{MC2} & \textbf{MC3} & \textbf{MC1} & \textbf{MC2} & \textbf{MC3} \\
\midrule
High   & 0.6591 & 0.4765 & 0.5699 & 0.6207 & 0.5448 & 0.6009 & 0.7234 & 0.5957 & 0.6799 \\
Middle & 0.6591 & 0.4765 & 0.5699 & 0.6207 & 0.5448 & 0.6009 & 0.7021 & 0.5862 & 0.6678 \\
Low    & 0.6591 & 0.4765 & 0.5699 & 0.6207 & 0.5448 & 0.6009 & 0.7660 & 0.6255 & 0.7103 \\
\bottomrule
\end{tabular}
}
\caption{Impact of selecting candidate layers from different depth layers of LRMs.}
\label{tab:attention-layer-ablation}
\end{table}

\begin{figure}[h]
    \centering
    \begin{subfigure}{0.47\textwidth}
        \includegraphics[width=\linewidth]{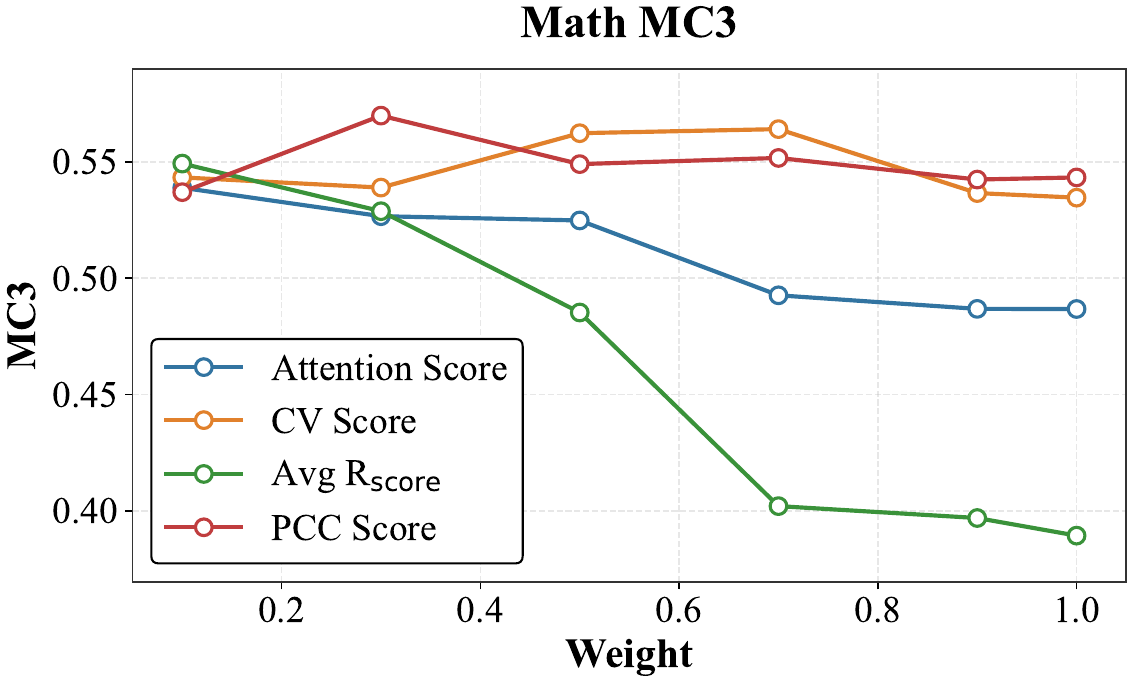}
        \caption*{ReTruthQA (Math)}
    \end{subfigure}
    \hfill
    \begin{subfigure}{0.47\textwidth}
        \includegraphics[width=\linewidth]{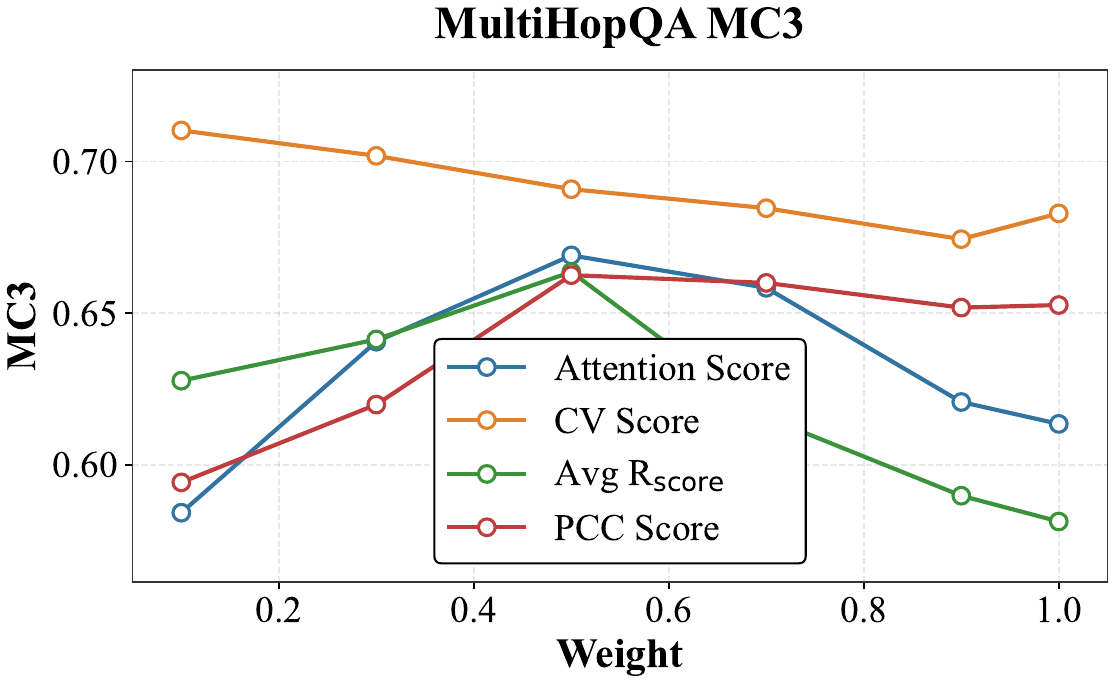}
        \caption*{ReTruthQA (MultiHopQA)}
    \end{subfigure}
    \caption{We conduct a sensitivity analysis of each module in RHD, using R1-7B on the Math and MultiHopQA subsets of ReTruthQA. We vary the weights assigned to different components and observe the resulting performance on the MC3 metric.}
    \label{fig:sen_analysis}
\end{figure}

\section{Sensitivity Analysis of RHD}
\label{app:sen_rhd}
In this section, we conduct sensitivity analysis experiments to investigate the impact of design choices in RHD. Inspired by the underlying reasoning mechanism, we fix the reasoning score to be extracted from the later layers of LRMs. Our primary focus is on selecting the appropriate layers for computing the attention score.
Specifically, we evaluate three different layer groups: shallow layers (1, 3, 5, 7, 9, 11, 13), middle layers (8, 10, 12, 14, 16, 18), and deep layers (14, 16, 18, 20, 22, 24, 26) on R1-7B. The experimental results are shown in Table~\ref{tab:attention-layer-ablation}. We observe that, across the \texttt{Math} and \texttt{Science} domains, the choice of attention layers has limited influence on final performance. In contrast, for the \texttt{MultiHopQA} domain, shallow layers yield stronger results, aligning with the mechanistic interpretation that earlier layers are primarily responsible for information transmission. Based on these findings, we select the shallow layers as candidate layers for computing the attention score.

We further perform sensitivity analysis on influential feature weights in RHD across domains. We vary the feature weights in \{0.1, 0.3, 0.5, 0.7, 0.9\}, and present the results in Figure~\ref{fig:sen_analysis}. We observe that most features exhibit an initial increase in performance followed by either a decline or stabilization. The limited variance across settings indicates that the model is not overly sensitive to individual hyperparameter values, demonstrating the robustness and stability of the RHD framework.

\section{Implementation Details for Reasoning Hallucination Mitigation}
\label{app:imp_rhm}

We fine-tune the models for reasoning hallucination mitigation using a RL framework with the following hyperparameters: batch size of 8, learning rate of $1.0 \times 10^{-6}$, and 1 training epoch. We enable gradient checkpointing to reduce memory usage. The model is configured with a maximum prompt length of 512 and a maximum completion length of 7680. For parameter-efficient tuning, we adopt LoRA with rank $r=16$ and $\alpha=16$, applied to all linear layers (\texttt{lora\_target\_modules=all-linear}). During each training step, we sample 16 generations per query.

The reward function is a weighted sum of three components: (1) an accuracy reward that combines a rule-based parser~\citep{math_verify} and LLM-as-a-Judge~\citep{lightman2023let} to determine correctness, addressing the issue where the final answer is correct but fails rule-based extraction (reward = 1 for correct, 0 for incorrect); (2) a \textbf{format reward} that ensures adherence to the required reasoning format \texttt{<think>\textbackslash n...\textbackslash n</think>\textbackslash n<answer>\textbackslash n...\textbackslash n</answer>} (reward = 1 if the format is correct, 0 otherwise); and (3) a \textbf{tag count reward} that softly encourages the inclusion of each of the four required tags (\texttt{<think>}, \texttt{</think>}, \texttt{<answer>}, \texttt{</answer>}) by assigning 0.25 for each tag present. The reward weights are set to 1.0, 0.1, and 0.1 for the accuracy, format, and tag count rewards, respectively.

For evaluation, we use the same accuracy-based metric as in training, and report results by averaging over four sampled generations per input. The fine-tuned model, \texttt{DeepSeek-R1-Distill-Qwen-1.5B}, is publicly available at \url{https://huggingface.co/deepseek-ai/DeepSeek-R1-Distill-Qwen-1.5B}.

\begin{figure}[h]
    \centering
    \begin{subfigure}{0.47\textwidth}
        \includegraphics[width=\linewidth]{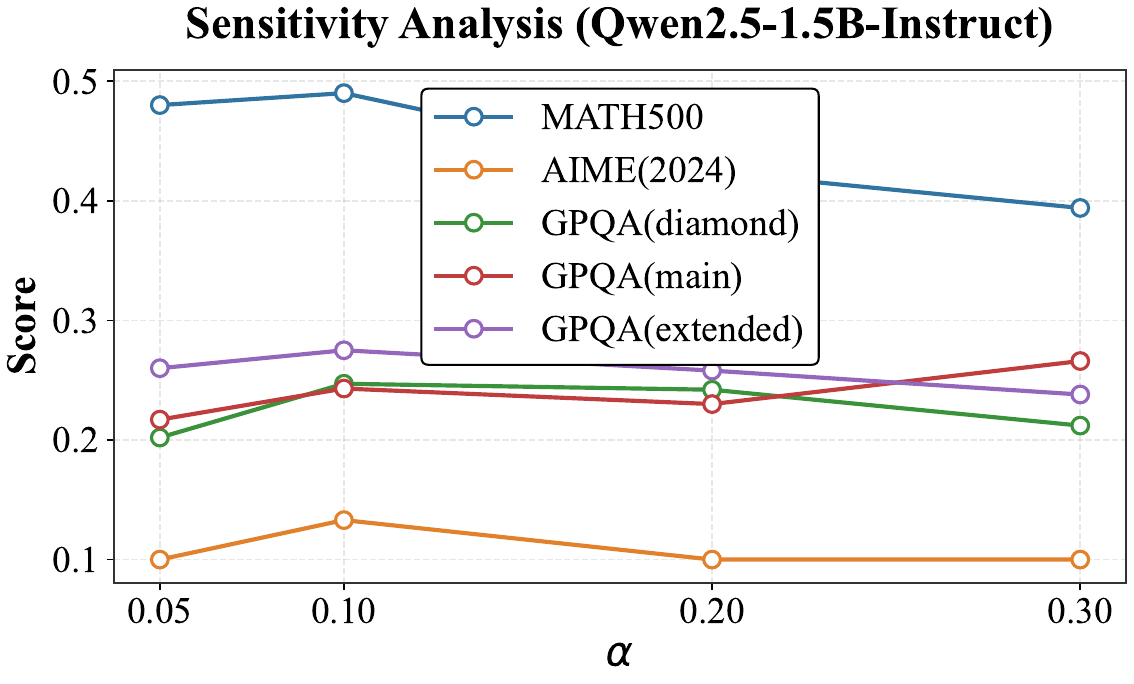}
    \end{subfigure}
    \hfill
    \begin{subfigure}{0.47\textwidth}
        \includegraphics[width=\linewidth]{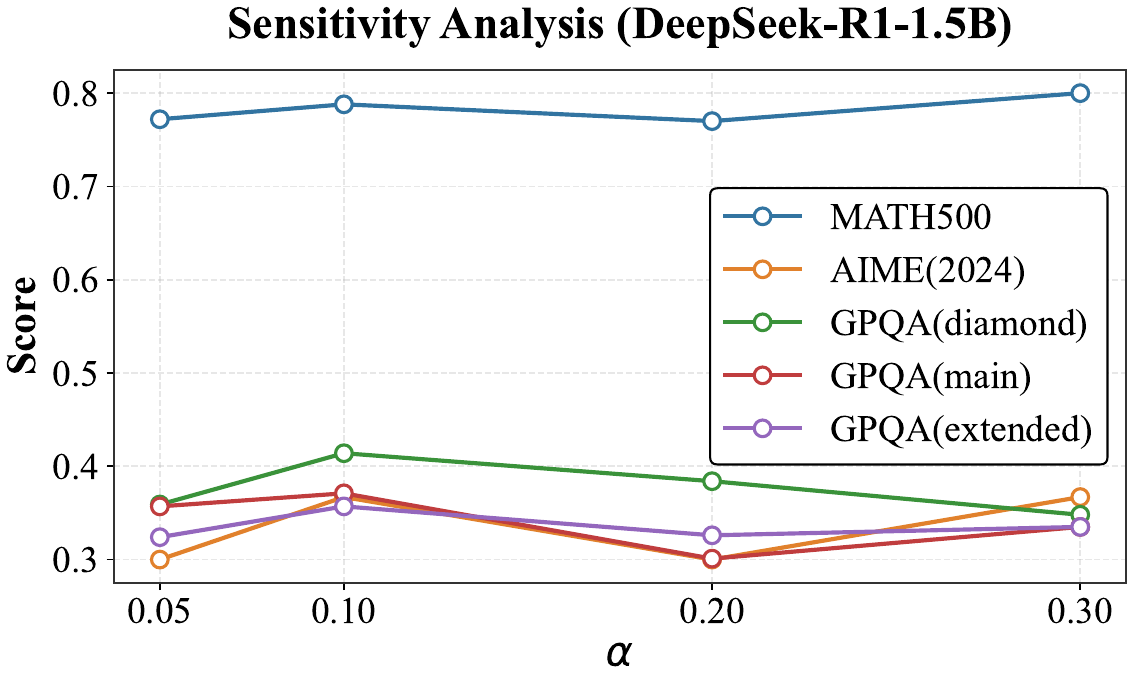}
    \end{subfigure}
    \caption{We conduct a sensitivity analysis on the weight of the reasoning score reward in GRPO-R, evaluating its impact on the accuracy metric. Experiments are carried out on both Qwen2.5-1.5B-Instruct and DeepSeek-R1-1.5B by varying the weight parameter $\alpha$.}
    \label{fig:sen_analysis_grpor}
\end{figure}

\section{Sensitivity Analysis of Reasoning Score Weight in GRPO-R}
\label{app:sen_rhm}
To investigate the sensitivity of the reasoning score reward weight $\alpha$ in the GRPO-R objective, we conduct experiments on both \texttt{DeepSeek-R1-1.5B} and \texttt{Qwen2.5-1.5B-Instruct}. We vary $\alpha$ in the range $[0.05, 0.1, 0.2, 0.3]$ and evaluate the models' performance accordingly.

Experimental results in \autoref{fig:sen_analysis_grpor} indicate that both models achieve the best average performance when $\alpha = 0.1$. As $\alpha$ increases beyond this value, we observe a gradual decline in performance. These results suggest that incorporating the reasoning score reward can effectively mitigate reasoning hallucinations without compromising accuracy, as long as it remains a secondary signal. However, overemphasizing the reasoning score (i.e., assigning it a large weight) can lead to a degradation in the model's ability to optimize for correctness, indicating that the reasoning signal should not dominate the outcome-based reward objective.

\begin{table}[t]
    \centering
    \caption{Accuracy of distilled models across benchmarks using different sampling strategies. Distillation is performed on Qwen2.5-1.5B-Instruct using reasoning traces from R1-14B.}
    \resizebox{\linewidth}{!}{
    \begin{tabular}{l|ccccc}
        \toprule
        \textbf{Method} & \textbf{MATH500} & \textbf{AIME (2024)} & \textbf{GPQA (diamond)} & \textbf{GPQA (main)} & \textbf{GPQA (extended)} \\
        \midrule
        Qwen2.5-1.5B-Instruct & 0.466 & 0.100 & 0.202 & 0.197 & 0.211 \\
        \midrule
        Random 20\%           & 0.504 & 0.100 & 0.247 & \textbf{0.230} & 0.242 \\
        RHD 20\%              & \textbf{0.520} & 0.100 & \textbf{0.263} & 0.210 & \textbf{0.249} \\
        \midrule
        Random 50\%      & 0.488 & 0.033 & 0.187 & 0.248 & \textbf{0.266} \\
        RHD 50\%              & \textbf{0.516} & \textbf{0.200} & \textbf{0.247} & \textbf{0.250} & 0.242 \\
        \midrule
        100\%       & 0.488 & 0.100 & 0.217 & 0.210 & 0.214 \\

        \bottomrule
    \end{tabular}
    }
    \label{tab:distill}
\end{table}

\section{RHD-Guided Reasoning Distillation}
\label{app:rhd_distill}
Distilling long-chain-of-thought data from large reasoning models to fine-tune smaller LLMs has become a widely adopted strategy for improving reasoning capabilities~\citep{DeepSeekAI2025}. However, directly fine-tuning small LLMs on raw LRM-generated data risks transferring undesirable reasoning behaviors such as shallow pattern matching or overthinking, potentially introducing reasoning hallucinations into the smaller models. To address this issue, we propose using the RHD score to rank distillation data and select more truthful samples for training.

The distillation setup uses a learning rate of $5.0 \times 10^{-5}$, batch size of 8, and LoRA applied to all linear layers with parameters \texttt{lora\_r} = 16 and \texttt{lora\_alpha} = 16. We use the training data from the hallucination mitigation experiment where \texttt{R1-14B} produces correct answers, along with their corresponding reasoning traces and final answers. We then score each reasoning trace using the RHD metric and sort the data in descending order. The top 20\% and 50\% of ranked samples are distilled into a smaller model, \texttt{R1-1.5B}, and compared against randomly sampled subsets of 20\%, 50\%, and 100\% of the same data.

Results, as shown in Table~\ref{tab:distill}, demonstrate that RHD-guided distillation consistently yields better performance across most evaluation benchmarks. In contrast, distillation using 100\% of the raw data results in degraded performance, likely due to noise introduced by hallucinated or low-quality samples. These findings validate the effectiveness of RHD in selecting high-quality data and mitigating reasoning hallucinations in downstream small LLMs during the distillation process.

\end{document}